\documentclass[letterpaper]{article} % DO NOT CHANGE THIS
\usepackage{aaai25}  % DO NOT CHANGE THIS

\usepackage{times}  % DO NOT CHANGE THIS
\usepackage{helvet}  % DO NOT CHANGE THIS
\usepackage{courier}  % DO NOT CHANGE THIS
\usepackage[hyphens]{url}  % DO NOT CHANGE THIS
\usepackage{graphicx} % DO NOT CHANGE THIS
\usepackage{natbib}  % DO NOT CHANGE THIS AND DO NOT ADD ANY OPTIONS TO IT
\usepackage{caption} % DO NOT CHANGE THIS AND DO NOT ADD ANY OPTIONS TO IT
\usepackage{algorithm}
\usepackage{algorithmic}

\usepackage{threeparttable}
\usepackage{newfloat}
\usepackage{listings}
\DeclareCaptionStyle{ruled}{labelfont=normalfont,labelsep=colon,strut=off} % DO NOT CHANGE THIS
\floatstyle{ruled}
\newfloat{listing}{tb}{lst}{}
\floatname{listing}{Listing}
\usepackage[utf8]{inputenc} % allow utf-8 input
\usepackage[T1]{fontenc}    % use 8-bit T1 fonts
\usepackage{url}            % simple URL typesetting
\usepackage{booktabs}       % professional-quality tables
\usepackage{amsfonts}       % blackboard math symbols
\usepackage{nicefrac}       % compact symbols for 1/2, etc.
\usepackage{microtype}      % microtypography
\usepackage{xcolor}         % colors
\usepackage{listings}
\usepackage{overpic}
\usepackage{graphicx}
\usepackage{colortbl}
\usepackage{makecell}
\usepackage{multirow}
\usepackage{tabularx}
\usepackage{verbatim}
\usepackage{svg}
\usepackage{subcaption}

\usepackage{amsthm}
\usepackage{amssymb}
\usepackage{amsmath}
\usepackage{amsfonts}       % blackboard math symbols
\usepackage{nicefrac}       % compact symbols for 1/2.
\usepackage{mathtools}

\usepackage{mathrsfs}

\newtheorem{defn}{Definition}

\usepackage{bm}
\newcommand{\x}{\mathbf{x}}
\newcommand{\F}{\mathbf{F}}
\newcommand{\G}{\mathbf{G}}
\newcommand{\bd}{\mathbf{d}}
\newcommand{\Blambda}{\bm{\lambda}}
\newcommand{\ee}{\mathbb{E}_{\xi}}
\newcommand{\mv}{\mathbb{V}_{\xi}}
\newcommand{\alg}{PSMGD }

\usepackage{thmtools}
\usepackage{thm-restate}
\newcommand{\mr}{{\textbf{MR}}}
\newcommand{\dm}{{\bm{$\Delta m\%$}}}

\newcommand{\best}[1]{{\textbf{\textcolor{contextcolor}{#1}}}}

\newcommand{\stl}{\textsc{STL}}
\newcommand{\ls}{\textsc{LS}}
\newcommand{\si}{\textsc{SI}}
\newcommand{\rlw}{\textsc{RLW}}
\newcommand{\dwa}{\textsc{DWA}}
\newcommand{\uw}{\textsc{UW}}

\newcommand{\famo}{\textsc{FAMO}}
\newcommand{\psmg}{\textsc{PSMGD}}
\newcommand{\graddrop}{\textsc{GradDrop}}

\newcommand{\mgda}{\textsc{MGDA}}
\newcommand{\pcgrad}{\textsc{PCGrad}}
\newcommand{\imtlg}{\textsc{IMTL-G}}
\newcommand{\cagrad}{\textsc{CAGrad}}
\newcommand{\nashmtl}{\textsc{NashMTL}}
\newcommand{\sdmgrad}{\textsc{SDMGrad}}
\newcommand{\fairgrad}{\textsc{FairGrad}}
\newcommand{\moco}{\textsc{MoCo}}

\definecolor{contextcolor}{RGB}{31,160,171}

\usepackage[textsize=tiny]{todonotes}

\usepackage{amsmath,amsfonts,bm}

\def\eqref#1{equation~\ref{#1}}
\def\1{\bm{1}}

\def\d{{\mathrm{d}}}

\def\rvx{{\mathbf{x}}}
\def\rvy{{\mathbf{y}}}

\DeclareMathAlphabet{\mathsfit}{\encodingdefault}{\sfdefault}{m}{sl}
\SetMathAlphabet{\mathsfit}{bold}{\encodingdefault}{\sfdefault}{bx}{n}

\DeclareMathOperator*{\argmin}{arg\,min}

\newif\ifcomments

\commentsfalse
\commentstrue

\ifcomments
  \newcommand{\colornote}[3]{{\color{#1}\bf{#2: #3}\normalfont}}
\else
  \newcommand{\colornote}[3]{}
\fi

\theoremstyle{plain}
\newtheorem{theorem}{Theorem}[section]

\newtheorem{lemma}[theorem]{Lemma}

\theoremstyle{definition}

\newtheorem{assumption}[theorem]{Assumption}
\theoremstyle{remark}
\newtheorem{remark}[theorem]{Remark}

\usepackage[capitalize]{cleveref}
\crefname{equation}{Eq.}{Eq.}
\crefname{figure}{Figure}{Figure~}
\crefname{table}{Table}{Table~}
\crefname{section}{Sec.}{Sec.~}
\crefname{algorithm}{Algorithm}{Algorithm~}
\crefname{thm}{Theorem}{Theorem~}
\crefname{lemma}{Lemma}{Lemma~}
\crefname{appendix}{Appendix}{Appendix~}

\usepackage{pifont}

\def\x{\rvx}
\def\y{\rvy}

\title{
PSMGD: Periodic Stochastic Multi-Gradient Descent for Fast Multi-Objective Optimization
}

\author{
    Mingjing Xu\textsuperscript{\rm 1},
    Peizhong Ju\textsuperscript{\rm 2},
    Jia Liu\textsuperscript{\rm 3},
    Haibo Yang\textsuperscript{\rm 1}
}
\affiliations{
    \textsuperscript{\rm 1} Dept. of Computing and Information Sciences Ph.D., Rochester Institute of Technology\\
    \textsuperscript{\rm 2} Dept. of Computer Science, University of Kentucky \\
    \textsuperscript{\rm 3}  Dept. of Electrical and Computer Engineering, The Ohio State University \\
    \{mx6835, hbycis\}@rit.edu, peizhong.ju@uky.edu, liu.1736@osu.edu

}

\begin{document}

\maketitle

% !TEX root = main.tex

\begin{abstract}

Multi-objective optimization (MOO) lies at the core of many machine learning (ML) applications that involve multiple, potentially conflicting objectives. 
Despite the long history of MOO, recent years have witnessed a surge in interest within the ML community in the development of gradient manipulation algorithms for MOO, thanks to the availability of gradient information in many ML problems. 
However, existing gradient manipulation methods for MOO often suffer from long training times, primarily due to the need for computing dynamic weights by solving an additional optimization problem to determine a common descent direction that can decrease all objectives simultaneously.
To address this challenge, we propose a new and efficient algorithm called Periodic Stochastic Multi-Gradient Descent (PSMGD) to accelerate MOO. 
PSMGD is motivated by the key observation that dynamic weights across objectives exhibit small changes under minor updates over short intervals during the optimization process. 
Consequently, our PSMGD algorithm is designed to periodically compute these dynamic weights and utilizes them repeatedly, thereby effectively reducing the computational overload.
Theoretically, we prove that PSMGD can achieve state-of-the-art convergence rates for strongly-convex, general convex, and non-convex functions.
Additionally, we introduce a new computational complexity measure, termed backpropagation complexity, and demonstrate that PSMGD could achieve an objective-independent backpropagation complexity.
Through extensive experiments, we verify that PSMGD can provide comparable or superior performance to state-of-the-art MOO algorithms while significantly reducing training time.
% Code is provided at \url{https://anonymous.4open.science/r/PSMG-CFE3}.
\end{abstract}

\begin{links}
  \link{Code}{https://github.com/MarcuXu/PSMG}
  % \link{Datasets}{https://aaai.org/example/datasets}
  \link{Extended version}{https://arxiv.org/abs/2412.10961}
\end{links}

\section{Introduction} \label{sec: intro}

Just as the human world is full of diverse and potentially conflicting goals, many machine learning paradigms are also multi-objective, such as multi-objective reinforcement learning~\citep{hayes2022practical}, multi-task learning~\citep{caruana1997multitask,sener2018multi}, and learning-to-rank~\citep{mahapatra2023multi}.
At the core of these machine learning paradigms lies multi-objective optimization (MOO), which aims to optimize multiple potentially conflicting objectives simultaneously.
Formally, MOO can be mathematically cast as:
\begin{align} \label{eq: moo}
    \min_{\x \in \mathcal{D}} \F(\x) := [f_1(\x), \cdots, f_S(\x) ],
\end{align}
where $\x \in \mathcal{D} \subseteq \mathbb{R}^d$ is the model parameter, and $f_s : \mathbb{R}^d \rightarrow \mathbb{R}$, $s \in [S]$ is the objective function.

The most common approach for handling MOO is to optimize a weighted average of the multiple objectives, also known as linear scalarization~\citep{kurin2022defense,xin2022current}. 
This approach transforms MOO into a single-objective problem, for which there are various off-the-shelf methods available.
The linear scalarization approach has a fast per-iteration runtime but suffers from slow convergence and poor performance due to potential conflicts among objectives.
% (see Sec~\ref{sec: relatedwork} for details).
% However, this static weighting approach often suffers from poor performance in practical usage and insufficiency to explore all solutions with different trade-offs due to potential conflicts among multiple objectives during the optimization process.
%
To address this problem, there has been a surge of interest in recent years in developing gradient manipulation algorithms~\citep{kendall2018multi,chen2018gradnorm,yu2020gradient,chen2020just,javaloy2021rotograd,liu2021stochastic,chen2024three,xiao2024direction}. The key idea is to calculate dynamic weights to avoid conflicts and find a direction $\bd(\x)$ that decreases all objectives simultaneously. 
This common descent vector $\bd(\x)$ is often determined by solving an additional optimization problem that involves all objective gradients. 
Updating the model with the common descent direction $\bd$ ensures a decrease in every objective.
While these approaches show improved performance, the extra optimization process requires a substantial amount of time. This is particularly evident when the number of objectives is large and the model parameters are high-dimensional, which could significantly prolong the MOO training process. 
Existing empirical evaluations consistently demonstrate that gradient manipulation algorithms typically necessitate much longer training times~\citep{kurin2022defense,liu2024famo} (also shown in Sec~\ref{sec: exp}).
To this end, a natural question arises: {\em Can we design efficient gradient manipulation algorithms for fast MOO with theoretical guarantees?}

In this paper, we present Periodic Stochastic Multi-Gradient Descent (PSMGD), a simple yet effective algorithm to accelerate the MOO. 
Our PSMGD is motivated by the key observation that dynamic weights across objectives exhibit small changes under minor updates over short intervals. 
Thus, it is designed to periodically compute these dynamic weights and utilizes them repeatedly, thereby effectively reducing the computational overload and achieving a faster MOO training process. Our contributions are summarized as follows:

\begin{table*}[h]
\begin{threeparttable}
\centering
\caption{Comparison of different algorithms for MOO problem in stochastic first order oracle.}
\begin{tabular}{|c|c|c|c|c|}
\hline
\multirow{2}{*}{\bf Convexity} & \multirow{2}{*}{\bf Algorithm} & \multirow{1}{*}{\bf Assume Lipschitz} & \multirow{1}{*}{\bf Convergence} & \multirow{2}{*}{\bf BP Complexity} \\
& & {\bf continuity of $\Blambda^*(\x)$} & {\bf Rate} & \\
\hline
\multirow{4}{*}{Strongly Convex} & SMG~\citep{liu2021stochastic} & \ding{51} & $\mathcal{O}(\frac{1}{T})$ & $\mathcal{O}(\frac{S}{\epsilon})$ \\
& MoDo~\citep{chen2024three} & \ding{55} & $\mathcal{O}(\frac{1}{T})$ & $\mathcal{O}(\frac{S}{\epsilon})$ \\
& CR-MOGM~\citep{zhou2022convergence} & \ding{55} & $\mathcal{O}(\frac{1}{T})$ & $\mathcal{O}(\frac{S}{\epsilon})$ \\
\rowcolor{gray!20} & \alg & \ding{55} & $\mathcal{O}(\frac{1}{T})$ & $\mathcal{O}(\frac{S}{\epsilon R} + \frac{(R-1)}{\epsilon R})$ \\
\hline
\multirow{3}{*}{General Convex} & SMG~\citep{liu2021stochastic} & \ding{51} & $\mathcal{O}(\frac{1}{\sqrt{T}})$ & $\frac{S}{\epsilon^2}$ \\
& CR-MOGM~\citep{zhou2022convergence} & \ding{55} & $\mathcal{O}(\frac{1}{\sqrt{T}})$ & $\frac{S}{\epsilon^2}$ \\
 \rowcolor{gray!20} & \alg & \ding{55} & $\mathcal{O}(\frac{1}{\sqrt{T}})$ & $\mathcal{O}(\frac{S}{\epsilon^2 R} + \frac{(R-1)}{\epsilon^2 R})$ \\
\hline
\multirow{5}{*}{Non-Convex} & CR-MOGM~\citep{zhou2022convergence} & \ding{55} & $\mathcal{O}(\frac{1}{\sqrt{T}})$ & $\frac{S}{\epsilon^2}$ \\
 & MoCO~\citep{fernando2022mitigating} & \ding{55} & $\mathcal{O}(\frac{1}{\sqrt{T}})$ & $\frac{S}{\epsilon^2}$ \\
 & MoDo~\citep{chen2024three} & \ding{55} & $\mathcal{O}(\frac{1}{\sqrt{T}})$ & $\frac{S}{\epsilon^2}$ \\
 & SDMGrad~\citep{xiao2024direction} & \ding{55} & $\mathcal{O}(\frac{1}{\sqrt{T}})$ & $\frac{S}{\epsilon^2}$ \\
 \rowcolor{gray!20} & \alg & \ding{55} & $\mathcal{O}(\frac{1}{\sqrt{T}})$ & $\mathcal{O}(\frac{S}{\epsilon^2 R} + \frac{(R-1)}{\epsilon^2 R})$ \\
\hline
\end{tabular}
\label{tab:rate}
\begin{tablenotes}
  \small
  \item 1.  BP Complexity measures the total number of backpropagation to achieve certain metric ($\epsilon$), which is defined in Def~\ref{def:BP}.
  \item 2. Notations: T is the number of iterations, S is the number of objectives, R is the weight calculation hyper-parameter.
\end{tablenotes}
\end{threeparttable}
\vspace{-0.1in}
\end{table*}

\begin{list}{\labelitemi}{\leftmargin=1em \itemindent=-0.0em \itemsep=.1em}
    \item Built upon the key observation, we introduce PSMGD,  a new and efficient gradient manipulation method. By calculating dynamic weights infrequently, it significantly alleviates computational burden and reduces training time.
    \item We conduct a theoretical analysis of PSMGD for strongly convex, convex, and non-convex functions, demonstrating that PSMGD achieves state-of-the-art convergence rates comparable to existing MOO methods (Table~\ref{tab:rate}). We also introduce a new computational complexity measure, backpropagation (BP) complexity, to quantify computational workload, and further show that PSMGD can achieve an objective-independent backpropagation complexity. 
    % Compared with existing MOO methods, this shows a superior BP complexity with the linear speedup with respect to the number of objectives.
    \item Through comprehensive evaluation, PSMGD shows comparable or even superior performance compared to existing MOO methods, with significantly less training time.
\end{list}

\section{Related works} \label{sec: relatedwork}

The MOO problem, as stated in Eq.~\ref{eq: moo}, has a long history that dates back to the 1950s. MOO algorithms can be broadly categorized into two main groups. The first category comprises gradient-free methods, such as evolutionary MOO algorithms and Bayesian MOO algorithms~\citep{zhang2007moea,deb2002fast,belakaria2020uncertainty,laumanns2002bayesian}. These methods are more suitable for small-scale problems but are less practical for high-dimensional models, such as deep neural networks. The second category is the gradient-based approach by utilizing (stochastic) gradients~\citep{fliege2000steepest,desideri2012multiple,fliege2019complexity,liu2021stochastic}, making them more practical for high-dimensional MOO problems. In this work, we primarily focus on gradient-based approaches to solve MOO in high-dimensional deep-learning models.
We first provide a brief overview of two typical approaches: linear scalarization (LS) and gradient manipulation methods.

\textbf{1) LS:}
A straightforward approach is to transform it into a single objective problem by using pre-defined weights $\Blambda$: $ \min_{\x \in \mathcal{D}} \Blambda^{\top} \F(\x).$
Thanks to the LS, one can leverage many existing single-objective methods (e.g., gradient descent), for training at each iteration. However, in LS, it is not uncommon to see conflicts among multiple objectives during the optimization process, i.e., $\left< \nabla f_s(\x), \nabla f_{s'}(\x) \right> < 0$. This implies that the update using static weights may decrease some objectives, while inevitably increasing others at the same time and leading to slow convergence and poor performance.

\textbf{2) Gradient Manipulation Methods:}
A popular alternative is to dynamically weight gradients across objectives to avoid such conflicts and obtain a direction $\bd$ to decrease all objectives simultaneously.
For example, multiple gradient descent algorithm (MGDA)~\citep{fliege2000steepest} seeks to find the $\bd$ by solving the following optimization problem given gradients:
$(\bd, \beta) \in \argmin_{\bd \in \mathbb{R}^S, \beta \in \mathbb{R}} \beta + \frac{1}{2} \| \bd \|^2, s.t., \nabla f_s(\x)^T \bd - \beta \leq 0, \forall s \in [S].$
By doing so, if $\x$ represents a first-order Pareto stationary point, then $(\bd, \beta) = (\mathbf{0}, 0)$. Otherwise, we can find the $\bd$ such that $\nabla f_s(\x)^T \bd \leq \beta < 0, s \in [S]$ as the solution.
Using such a non-conflicting
direction $\bd$ to update the model, i.e., $\x \leftarrow \x + \eta \d$ where $\eta$ is the learning rate, has been shown to achieve better performance in practice~\citep{sener2018multi,liu2021conflict}.
Following this token, many gradient manipulation methods attempt to obtain a common descent direction through various formulations and solutions. 
For example, \graddrop{}~\citep{chen2020just} randomly dropped out highly conflicted gradients, RotoGrad~\citep{javaloy2021rotograd} rotated objective gradients to alleviate the conflict, and many other methods exploring similar principles~\citep{kendall2018multi,chen2018gradnorm,yu2020gradient,liu2021conflict,liu2021stochastic,chen2024three,xiao2024direction}.
However, due to the additional optimization process required to generate such $\bd$, gradient manipulation methods often incur an expensive per-iteration cost, thereby prolonging training time. In this work, we propose a new algorithm, PSMGD, which achieves fast per-iteration and overall convergence, along with enhanced performance.

\section{Periodic stochastic multi-gradient descent}
\label{sec: moo}
In this section, we first present the basic concept of MOO, followed by a pedagogical example comparing linear scalarization with MGDA, a representative of gradient manipulation methods. 
Based on one observation of the stable variation of dynamic weights for MGDA, we propose the PSMGD algorithm, followed by its convergence analyses.

\subsection{Preliminaries of MOO}
Analogous to the stationary and optimal solutions in single-objective optimization, MOO seeks to adopt the notion of Pareto optimality/stationarity:
\begin{defn} [(Weak) Pareto Optimality]
    For any two solutions $\x$ and $\y$, we say $\x$ dominates $\y$ if and only if $f_s(\x) \leq f_s(\y), \forall s \in [S]$ and $f_s(\x) < f_s(\y), \exists s \in [S]$. 
    A solution $\x$ is Pareto optimal if it is not dominated by any other solution.
    One solution $\x$ is weakly Pareto optimal if there does not exist a solution $\y$ such that $f_s(\x) > f_s(\y), \forall s \in [S]$.
\end{defn}

\begin{defn} [Pareto Stationarity] \label{defn:ParetoStationarity}
    A solution $\x$ is said to be Pareto stationary if there is no common descent direction $\bd \in \mathbb{R}^d$ such that $\nabla f_s(\x)^{\top} \bd < 0, \forall s \in [S]$.
\end{defn}

Similar to solving single-objective non-convex optimization problems, finding a Pareto-optimal solution in MOO is NP-Hard in general.
As a result, it is often of practical interest to find a solution satisfying Pareto-stationarity (a necessary condition for Pareto optimality).
Following Definition~\ref{defn:ParetoStationarity}, if $\x$ is not a Pareto stationary point, we can find a common descent direction $\bd \in \mathbb{R}^d$ to decrease all objectives simultaneously, i.e., $\nabla f_s(\x)^{\top} \bd < 0, \forall s \in [S]$.
If no such a common descent direction exists at $\x$, then $\x$ is a Pareto stationary solution.
For example, multiple gradient descent algorithm (MGDA)~\citep{desideri2012multiple} searches for an optimal weight $\boldsymbol{\lambda}_*$ of gradients $\nabla \F(\x) := \{ \nabla f_s(\x), \forall s \in [S] \}$ by solving 
$\boldsymbol{\lambda}_*(\x) = \operatorname*{argmin}_{\boldsymbol{\lambda}} \| \boldsymbol{\lambda}^{\top} \nabla \F(\x) \|^2$, which is the dual of the original problem.
Then, a common descent direction can be chosen as: $\bd = \boldsymbol{\lambda}_*^{\top} \nabla \F(\x)$.
MGDA performs the iterative update rule: $\x \leftarrow \x - \eta \bd$ until a Pareto optimal/stationary point is reached, where $\eta$ is a learning rate. 
Many gradient manipulation algorithms have been inspired by MGDA. In the next subsection, we compare two typical gradient-based approaches using a pedagogical example, with MGDA representing the gradient manipulation approach.

\begin{figure*}[t!]
    \centering
    \includegraphics[width=\textwidth]{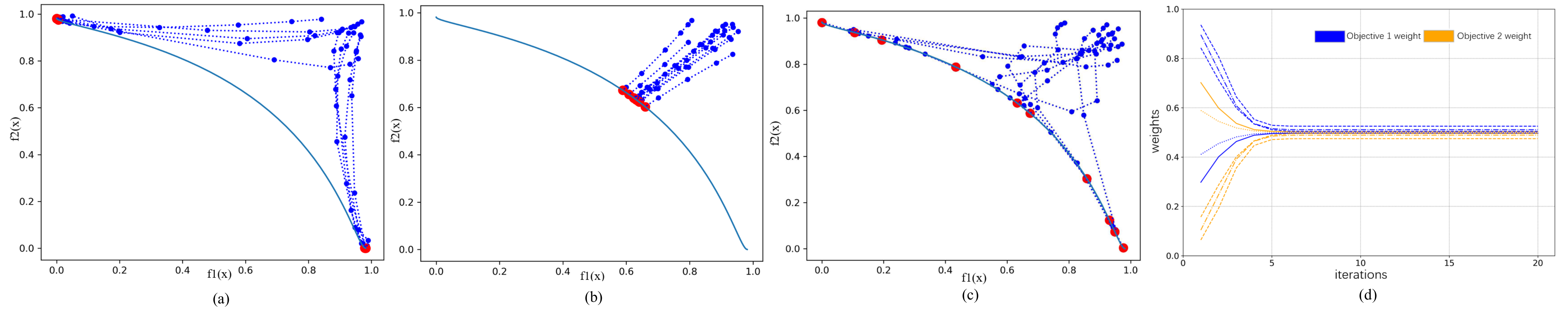}

    \caption{The convergence behaviors on a pedagogical example with 10 runs for each algorithm: (a) Solutions obtained from random linear scalarization. (b) Solutions obtained from the MGDA method. (c) Solutions obtained from the PSMGD method proposed in this paper, with weights updated every 4 iterations ($R=4$) through the training. (d) We visualize weights changing curves over iterations in 5 runs. This method successfully generates a set of widely distributed Pareto solutions with different trade-offs. Details of the pedagogical example can be found in Section~\ref{sec: syn}.
    }
    \label{fig:syn}
\end{figure*}

\subsection{A pedagogical example} \label{sec: syn}
We consider the Fonseca problem\citep{fonseca1996performance,pardalos2017non}, which is to solve a two-objective optimization problem, $\min_{\x} [f_1(\x), f_2(\x)]$. 
These two objective functions are defined by
$f_{1}(\x)=1-\exp \left(-\sum_{i=1}^{d}\left(x_{d}-\frac{1}{\sqrt{d}}\right)^{2}\right), f_{2}(\x)= 1-\exp \left(-\sum_{i=1}^{d}\left(x_{d}+\frac{1}{\sqrt{d}}\right)^{2}\right),$
where $\x = (x_1, x_2, ..., x_d)$ representing the d-dimensional decision variable. The problem exhibits a non-convex Pareto front in the objective space.
We run three algorithms on this pedagogical example: linear scalarization, MGDA, and our proposed \psmg{}.

{\em Setting.} 
We run each algorithm 10 times. In Figure~\ref{fig:syn}, we show the trajectory of different methods from random starting points. The blue dashed lines represent the trajectories, with the final solutions highlighted by red dots.

{\em Observation 1: Linear scalarization can only converge to specific points and is unable to explore different trade-offs among objectives. }
we run the linear scalarization method by using randomly static weights.
As shown in Figure~\ref{fig:syn}(a), it is obvious that linear scalarization can only converge to two specific points (see red dots), with either $f_1(\x)$ or $f_2(\x)$ being under-optimized.
In other words, linear scalarization cannot fully explore the Pareto front, which is consistent with previous results~\citep{miettinen1999nonlinear,pardalos2017non}. In contrast, MGDA can find trade-off points, and our proposed PSMGD successfully identifies a set of well-distributed Pareto solutions with different trade-offs, shown in Figure~\ref{fig:syn}(b,c).
We note that linear scalarization usually runs faster for MOO in each iteration than gradient manipulation methods due to its simplicity, as shown in many existing works~\citep{kurin2022defense,liu2024famo} and confirmed by our experimental results (see Sec~\ref{sec: exp}). However, our observation highlights the necessity of gradient manipulation methods and motivates us to develop faster gradient manipulation techniques for MOO.

{\em Observation 2: The dynamic weights calculated by MGDA converge quickly and exhibit stability.} 
As shown in Figure~\ref{fig:syn}(d), we visualize the dynamic weights for the Fonseca problem. 
In this simple problem, the weights converge within 5 iterations, after which they remain unchanged. 
In more complex problems, such as multi-task learning using deep learning (see Sec~\ref{sec:sup_ex} in Appendix), the dynamic weights calculated by gradient manipulation methods also demonstrate stability.
We verified stability by examining two scales, epoch and iteration. The weights show consistent behavior with a stable average value and low variability across both scales, epoch and iteration, as depicted in Figure~\ref{fig:weights}(a) and~\ref{fig:weights}(b) in Appendix.
This observation suggests a new approach to develop more efficient MOO algorithms. 
Specifically, we can calculate and update the dynamic weights periodically in a more lazy and thus efficient manner.

\subsection{PSMGD}

Based on our observations, we propose a new and efficient algorithm called Periodic Stochastic Multi-Gradient Descent (PSMGD) for MOO, as shown in Algorithm~\ref{algo}.
Our PSMGD algorithm is designed to periodically compute these dynamic weights and utilizes them repeatedly, thereby reducing the computational load effectively.
Specifically, in each iteration $t \in [T]$, we offer two options:
1) If $t \% R == 0$, we calculate the dynamic weights $\hat{\Blambda}_t^{*}$ based on the stochastic gradient of each objective $\nabla f_s(\x_t, \xi_t), s \in [S]$, by solving optimization problem ~\ref{eq: QP}. 
Following this, we apply momentum to $\Blambda$ to further stabilize the weights: $\Blambda_t = \alpha_t \Blambda_{t-R} + (1 - \alpha_t) \hat{\Blambda}_t^{*}$.
2) Otherwise, we reuse the dynamic weights: $\Blambda_t = \Blambda_{t-1}.$
With predefined $\Blambda_t$, it is worth pointing out that the MOO can be naturally transformed into a single-objective problem by using pre-defined weighted sum of the objectives, thereby requiring only one backpropagation step.
Subsequently, after obtaining the weights across objectives, we can approximate the common descent direction by $\bd_t = \Blambda_t^T \nabla \F(\x_t, \xi_t)$.
Then the model can be updated by $\x_{t+1} = \x_t -\eta_t \bd_t$.

\begin{algorithm}[t]
    \caption{Periodic Stochastic Multi-Gradient Descent (PSMGD)}
    \label{algo}
    \begin{algorithmic}[1]
        \STATE Initialize model parameter $\x_0$, learning rate $\eta$, and hyper-parameter $R$.
        \FOR{$t=0, \dots , T-1$}
            \STATE If $t \% R == 0$: \hfill  $\blacktriangleright$ Weights Calculation 
            \STATE \hspace{0.1in} Compute $\hat{\Blambda}_t^* \in [0, 1]^S$ by solving
                \begin{align}
                    &\min_{\Blambda} \Big\| \sum\nolimits_{s \in [S]} \lambda_s \nabla f_s(\x_t, \xi_t) \Big\|^2, \nonumber \\
                    &\quad \text{s.t.} \quad  \sum\nolimits_{s \in [S]} \lambda_s = 1. \label{eq: QP}
                \end{align} 
            \STATE \hspace{0.1in} Update: $\Blambda_t = \alpha_t \Blambda_{t-R} + (1 - \alpha_t) \hat{\Blambda}_t^*.$
            \STATE Otherwise: $\Blambda_t = \Blambda_{t-1}.$ \hfill  $\blacktriangleright$ Reuse the weights
            
            \STATE Update the model: $\x_{t+1} = \x_t -\eta_t \bd_t$, where $\bd_t = \sum\nolimits_{s \in [S]} \lambda_{t, s} \nabla f_s(\x_t, \xi_t)$.
        \ENDFOR
\end{algorithmic} 
\end{algorithm}  

\subsection{Convergence analysis} \label{sec:theory}
% Before providing the convergence analysis, we have the following assumptions:

\begin{assumption}[L-Lipschitz continuous]
\label{assum:LSmooth}
% There exists a constant $L>0$ such that 
$\| \nabla f_s(\x) - \nabla f_s(\y) \| \leq L \| \x - \y \|, \forall \x, \y \in \mathbb{R}^d, s \in [S]$.
\end{assumption}

\begin{assumption}
[Bounded variance]
\label{assum:sgd}
We assume the stochastic gradient estimation is unbiased with bounded variance.
\begin{align}
&\mathbb{E} \left[ \nabla f_s(\x, \xi) \right] = \nabla f_s(\x), \\
&\mathbb{E} \left[ \| \nabla f_s(\x, \xi) - \nabla f_s(\x) \|^2 \right] \leq \sigma^2, s \in [S].
\end{align}
\end{assumption}

\begin{assumption}
[Bounded weights]
\label{assum:bw}
There exists a constant $B$ s.t., $0 \leq \lambda_{t, s} \leq B,  \sum_{s \in [S]} \lambda_{t, s} \geq 1, \forall s \in [S], t \in [T]$.
\end{assumption}

\begin{assumption}[Bounded Gradient]
\label{assum:BG}
       The gradient of each objective is bounded, i.e., there exists a constant $H>0$ such that $\| \nabla f_s(\x) \| \leq H, \forall s \in [S]$.
\end{assumption}

These four assumptions are widely-used assumptions in MOO~\citep{liu2021stochastic,zhou2022convergence}.
For notation clarity, we have the following definition:
$G(\x_t, \Blambda_t) = \sum_{s \in [S]} \lambda_{t, s} f_s(\x_t), \nabla \G(\x_t, \Blambda_t) = \sum_{s \in [S]} \lambda_{t, s} \nabla f_s(\x_t)$.
With these assumptions and definitions, we have the following convergence rates:

%no enough lines there, so there is a gap%

% \begin{assumption}[Bounded Function]
% \label{assum:BF}
%        The value of each objective is bounded, i.e., there exists a constant $F>0$ such that $\| f_s(\x) \| \leq F, \forall s \in [S]$.
% \end{assumption}
% Each objective $f_s(\x), s \in [S]$ is a $\mu$-strongly convex function, i.e., $f_s(\y) \geq f_s(\x) + \nabla f_s(\x) (\y - \x) + \frac{\mu}{2} \| \y - \x \|^2$ for some $\mu >0$.

\begin{restatable}[Non-Convex Functions] {theorem}{pmgdNC}
\label{thm:pmgda_NC}
Under Assumptions~\ref{assum:LSmooth}-~\ref{assum:BG}, when each objective is bounded by $F$ ($f_s(\x) \leq F, s \in [S]$), the sequence of iterates generated by the \alg Algorithm in non-convex functions satisfies:
% \allowdisplaybreaks
% \begin{equation}
% \begin{split}
% \begin{align}
%     \frac{1}{T} \sum_{t \in [T]} \mathbb{E} [\| \nabla \G(\x_t, \Blambda_t) \|^2] 
%     &\leq \frac{4SBF}{T \eta_T} + \frac{2F}{T} \sum_{t \%R == 0, t \neq 0} \frac{ \mathbb{E} \left[ \sum_{s \in [S]} (1 - \alpha_t ) | \hat{\lambda}_{t}^s - \lambda_{t-R}^s | \right]}{\eta_t} \nonumber \\
%     &\quad + 2LS^2B^2 \sigma^2 \frac{1}{T} \sum_{t \in [T]} \eta_t + \frac{4S^{3/2}B^2H \sigma}{T} \sum_{t \% R == 0} (1 - \alpha_t).
% \end{align}
% \end{split}
% \end{equation}
\begin{align}
    &\frac{1}{T} \sum_{t \in [T]} \mathbb{E} [\| \nabla \G(\x_t, \Blambda_t) \|^2]  \\
    &\leq \frac{4SBF}{T \eta_T} + \frac{2F}{T} \sum_{\substack{t \%R == 0 \\ t \neq 0}} \frac{ \mathbb{E} \left[ \sum_{s \in [S]} (1 - \alpha_t ) | \hat{\lambda}_{t}^s - \lambda_{t-R}^s | \right]}{\eta_t}  \nonumber \\
    &+ 2LS^2B^2 \sigma^2 \frac{1}{T} \sum_{t \in [T]} \eta_t + \frac{4S^{3/2}B^2H \sigma}{T} \sum_{t \% R == 0} (1 - \alpha_t). \nonumber
\end{align}

% \begin{equation}
% \begin{split}
% F = \{F_{x} \in  F_{c} &: (|S| > |C|) \\
% &\quad \cap (\text{minPixels} < |S| < \text{maxPixels}) \\
% &\quad \cap (|S_{\text{conected}}| > |S| - \epsilon) \}
% \end{split}
% \end{equation}

Setting $1 - \alpha_t = \min \{\frac{\eta_t}{\eta_1}, \frac{\eta_t}{\eta_1 \sqrt{t} \max_{s \in [S]} | \hat{\lambda}_{t}^s - \lambda_{t-R}^s |} \}, \eta_t = \mathcal{O}(\frac{1}{\sqrt{t}})$, the convergence rate is $\mathcal{O}(\frac{1}{\sqrt{T}}).$
% \begin{align}
%     \frac{1}{T} \sum_{t \in [T]} \mathbb{E} \| \nabla \G(\x_t, \Blambda_t) \|^2 = \mathcal{O}(\frac{1}{\sqrt{T}}).
% \end{align}
\end{restatable}

\begin{restatable}[General Convex Functions] {theorem}{pmgdGC}
\label{thm:pmgda_GC}
Under Assumptions~\ref{assum:LSmooth}-~\ref{assum:BG}, when the distance from sequence to Pareto set
is bounded ($\| \x_t - \x_* \| \leq D$), the sequence of iterates generated by the \alg in general convex functions satisfies:

\begin{align}
    &\frac{1}{T} \sum_{t \in [T]} \mathbb{E} [G (\x_t, \Blambda_t) - G (\x_*, \Blambda_t)] \\
    &\leq \frac{\| \x_1 - \x_* \|^2}{T \eta_1} + 2 D \sigma S^{3/2} B \frac{1}{T} \sum_{t \% R == 0} (1 - \alpha_t) \nonumber \\
    &\quad + (2 S^2B^2 \sigma^2 + 2 S^2 B^2 H^2) \frac{1}{T} \sum_{t \in [T]} \eta_t. \nonumber
\end{align}

Setting $\eta_t = \frac{1}{\sqrt{t}}$ and $(1 - \alpha_t) = \eta_t$, the convergence rate is $\mathcal{O}(1/\sqrt{T}).$
\end{restatable}

\begin{restatable}[$\mu$-Strongly Convex Functions] {theorem}{pmgdSC}
\label{thm:pmgda_SC}
Under Assumptions~\ref{assum:LSmooth}-~\ref{assum:BG} and we further assume each objective function $f_s(\x), s \in [S]$ is $\mu-$strongly convex and each objective is bounded by $F$ ($f_s(\x) \leq F, s \in [S]$), the sequence of iterates generated by the \alg satisfies:
\begin{align}
    & \mathbb{E} [G(\x_{t+1}, \Blambda_{t+1}) - G(\x_{t+1}^*, \Blambda_{t+1})] \nonumber \\
    &\leq (1 - 2 \mu \eta_t) \mathbb{E} [G(\x_t, \Blambda_t) - G(\x_t^*, \Blambda_t)] + \eta_t^2 \Phi,
\end{align}
where $\Phi = 2L S^2B^2 \sigma^2  + 4 F SB \mathbf{1} \{(t+1) \% R = 0\} + 4S^{3/2}B^2H \sigma \mathbf{1} \{t \% R = 0\}$ and $\mathbf{1}$ is the indicator function.
Set $\eta = \frac{c}{T}$ with $c > \frac{1}{\mu}$, the convergence rate is
\begin{align}
    &\mathbb{E} [G(\x_T, \Blambda_T) - G(\x_T^*, \Blambda_T)] \nonumber \\
    &\leq \frac{\max \{2 c^2 \Phi^{'} (2\mu c - 1)^{-1}, G(\x_0, \Blambda_0) - G(\x_0^*, \Blambda_0) \}}{T} \nonumber \\
    & = \mathcal{O}(1/T),
\end{align}
where $\Phi^{'} = 2L S^2B^2 \sigma^2  + 4 F SB/R + 4S^{3/2}B^2H \sigma /R.$
\end{restatable}

\begin{remark}
1) Convergence Metrics. Different convergence metrics are used in existing studies, and we follow convergence conditions outlined in \citep{tanabe2019proximal,zhou2022convergence}. In the appendix, we provide an explanation of convergence metrics and a comparison with other works.
2) Convergence Rates. With proper hyperparameters, our algorithm, \alg, can achieve $\mathcal{O}(1/T)$ for strongly convex functions and $\mathcal{O}(1/\sqrt{T})$ for general convex and non-convex functions. These convergence rates match the state-of-the-art rates in existing MOO algorithms. 
% which are also validated by extensive experiments in Sec~\ref{sec: exp}. 
% We provide a comparison of these rates in Table~\ref{tab:rate} in Appendix.
\end{remark}

The convergence rate reflects the training speed in terms of iterations but does not fully capture the total computational complexity. Similar to how sample complexity is used in single-objective learning, we propose a new metric, Backpropagation (BP) complexity, to quantify the computational workload tailored to MOO in first-order oracle.
\begin{defn}[Backpropogation complexity] \label{def:BP}
We define Backpropagation Complexity as the total number of backpropagation operations required by an algorithm to achieve a specified performance threshold, denoted as $\epsilon$.
\end{defn}
\begin{remark}
Our PSMGD algorithm can achieve BP complexity of $\mathcal{O}(\frac{S}{\epsilon R} + \frac{(R-1)}{\epsilon R})$ for stongly-convex functions and $\mathcal{O}(\frac{S}{\epsilon^2 R} + \frac{(R-1)}{\epsilon^2 R})$ for general convex and non-convex functions.
Compared to existing MOO methods, PSMGD can achieve a linear speedup in terms of $R$. 
\end{remark}

\begin{remark}
If $R = \Omega(S)$, PSMGD exhibits an {\em objective- independent BP complexity}, i.e., $\mathcal{O}(\frac{1}{\epsilon}) $ for stongly-convex functions and $\mathcal{O}(\frac{1}{\epsilon^2} )$ for general convex and non-convex functions. 
These rates indicate that PSMGD requires the same order of computations for MOO as classic SGD does for single-objective learning.
We provide a comparison of these rates and complexities in Table~\ref{tab:rate}, which are also validated by extensive experiments in Sec~\ref{sec: exp}.
\end{remark}

\section{Experiment} \label{sec: exp}

In this section, we conduct a comprehensive empirical evaluation of the proposed \psmg{} algorithm, focusing primarily on multi-task learning with deep neural networks.
% \footnote{Code: \url{https://anonymous.4open.science/r/PSMG-CFE3}.}
The primary goal is to verify the following key points:

\begin{list}{\labelitemi}{\leftmargin=1em \itemindent=-0.0em \itemsep=.0em}
\item Improved Performance: Can \psmg{} provide improved model performance?
\item Fast Training: Can \psmg{} have a fast training process?
\item Ablation Study: How does the hyper-parameter impact the training? (Results In Appendix)
\end{list}

\textbf{Baseline Algorithms.}
To conduct an extensive comparison, we use 14 algorithms as the baselines, including a single objective learning baseline, linear scalarization, and 12 advanced MOO algorithms.
They are: \textbf{(1)} Single task learning (\stl{}), training independent models $f_S(\x)$ for all objectives; \textbf{(2)} Linear scalarization (\ls{}) baseline; \textbf{(3)} Scale-invariant (\si{}) that minimizes $\sum_{s \in [S]} \log f_s(\x)$; \textbf{(4)} Dynamic Weight Average (\dwa{})~\citep{liu2019end} adaptively updates weights based on the comparative rate of loss reduction for each objective; \textbf{(5)} Uncertainty Weighting (\uw{})~\citep{kendall2018multi} leverages estimated task uncertainty to guide weight assignments; \textbf{(6)} Random Loss Weighting (\rlw{})~\citep{lin2021closer} samples objective weighting with log-probabilities by normal distribution; \textbf{(7)} \mgda{}~\citep{sener2018multi} identifies a joint descent direction that concurrently decreases all objectives; \textbf{(8)} \pcgrad{}~\citep{yu2020gradient} projects each objective gradient onto the normal plan of the gradients to avoid conflicts; \textbf{(9)} \cagrad{}~\citep{liu2021conflict} balances the average loss while guaranteeing a controlled minimum improvement for each objective; \textbf{(10)} \imtlg{}~\citep{liu2021towards} determines the update direction by having equal projections on objective gradients; \textbf{(11)} \graddrop{}~\citep{chen2020just} randomly drops out certain dimensions of the objective gradients for their level of conflict; \textbf{(12)} \nashmtl{}~\citep{navon2022multi} determines the game's solution that is advantageous for all goals by a bargaining game; \textbf{(13)} \famo{}~\citep{liu2024famo} is a dynamic weighting method that reduces objective losses in space and time in a balanced manner; \textbf{(14)} \fairgrad{}~\citep{maheshwari2022fairgrad} aims to achieve group fairness in MOO by dynamically re-weighting; \textbf{(15)} \sdmgrad{}~\citep{xiao2024direction} utilizes direction-oriented regularization.

\textbf{Datasets.}
We evaluate on 5 multi-task learning datasets with objective/task numbers ranging from 2 to 40, covering scenarios in regression, classification, and dense prediction.
For regression, we choose QM-9 dataset~\citep{blum2009970} (11 tasks), a widely used benchmark in graph neural network learning. 
For image classification, we use Multi-MNIST~\citep{sener2018multi} (2 tasks) and CelebA dataset~\citep{liu2015deep} (40 tasks). 
% The former is a classical 2-task supervised learning benchmark, and the latter is an extensive face attribute dataset containing more than 200K celebrity images annotated with 40 different attributes. 
For dense prediction, CityScapes ~\citep{cordts2016cityscapes} (2 tasks) and NYU-v2~\citep{silberman2012indoor} (3 tasks) are used in our experiments. NYU-v2 is a dataset of indoor scenes with 1449 RGBD images and dense per-pixel labeling across 13 classes. 
% Objectives involve image segmentation, depth prediction, and surface normal prediction from scene images. 
CityScapes is similar to NYU-v2 but boasts 5K street-view RGBD images with per-pixel annotations.
We present the results for QM-9 and NYU-v2 in this section, with the remaining results and experimental settings detailed in Appendix~\ref{sec:sup_ex}.

\begin{table*}[htbp]
    \centering
    \resizebox{\textwidth}{!}{%
    \begin{tabular}{lrrrrrrrrrrrrr}
    \toprule
    \multirow{2}{*}{\textbf{Method}} & $\mu$ & $\alpha$ & $\epsilon_\text{HOMO}$ & $\epsilon_\text{LUMO}$ & $\langle R^2\rangle$ & ZPVE & $U_0$ & $U$ & $H$ & $G$ & $c_v$ &\multirow{2}{*}{\mr{} $\downarrow$} & \multirow{2}{*}{ \dm{} $\downarrow$}\\
      \cmidrule(lr){2-12}
      & \multicolumn{11}{c}{MAE $\downarrow$} & & \\
    \midrule 
        \stl{}      & 0.07 & 0.18 &  60.6 &  53.9 & 0.50  &  4.53  &  58.8  &  64.2 &  63.8 &  66.2 & 0.07 &     &       \\
        \midrule
        \ls{}       & 0.11 & 0.33 &  \best{73.6} &  89.7 & 5.20  & 14.06  & 143.4  & 144.2 & 144.6 & 140.3 & 0.13 & 9.00 & 177.6 \\
        \si{}       & 0.31 & 0.35 & 149.8 & 135.7 & \best{1.00}  &  \best{4.51}  &  \best{55.3}  &  \best{55.8} &  \best{55.8} &  \best{55.3} & 0.11 & 5.36 &  77.8 \\
        \rlw{}      & 0.11 & 0.34 &  76.9 &  92.8 & 5.87  & 15.47  & 156.3  & 157.1 & 157.6 & 153.0 & 0.14 & 10.36 & 203.8 \\
        \dwa{}      & 0.11 & 0.33 &  74.1 &  90.6 & 5.09  & 13.99  & 142.3  & 143.0 & 143.4 & 139.3 & 0.13 & 8.64 & 175.3 \\
        \uw{}       & 0.39 & 0.43 & 166.2 & 155.8 & 1.07  &  4.99  &  66.4  &  66.8 &  66.8 &  66.2 & 0.12 & 6.64 & 108.0 \\
        \mgda{}     & 0.22 & 0.37 & 126.8 & 104.6 & 3.23  &  5.69  &  88.4  &  89.4 &  89.3 &  88.0 & 0.12 & 8.36 & 120.5 \\
        \pcgrad{}   & 0.11 & 0.29 &  75.9 &  88.3 & 3.94  &  9.15  & 116.4  & 116.8 & 117.2 & 114.5 & 0.11 & 7.18 & 125.7 \\
        \cagrad{}   & 0.12 & 0.32 &  83.5 &  94.8 & 3.22  &  6.93  & 114.0  & 114.3 & 114.5 & 112.3 & 0.12 & 8.18 & 112.8 \\
        \imtlg{}    & 0.14 & 0.29 &  98.3 &  93.9 & 1.75  &  5.70  & 101.4  & 102.4 & 102.0 & 100.1 & 0.10 & 6.64 &  77.2 \\
        \nashmtl{}  & \best{0.10} & \best{0.25} &  82.9 &  81.9 & 2.43  &  5.38  &  74.5  &  75.0 &  75.1 &  74.2 & \best{0.09} & \best{3.82} &  62.0 \\
        \famo{}     & 0.15 & 0.30 & 94.0 & 95.2 &  1.63 & 4.95 & 70.82 & 71.2 & 71.2 & 70.3 & 0.10 & 5.09 & 58.5 \\
        \fairgrad{}     & 0.12 & \best{0.25} & 87.57 & 84.00 &  2.15 & 5.07 & 70.89 & 71.17 & 71.21 & 70.88 & 0.10 & 4.09 & \best{57.9} \\
    \midrule
    \psmg{}     & 0.12 & \best{0.25} & 77.2 & \best{74.4} & 3.01 & 6.61 & 103.0 & 103.5 & 103.7 & 101.6 & \best{0.09} & 5.55 & 92.4 \\
    \bottomrule 
    \end{tabular}
    }
    \vspace{2pt}
    \caption{Results obtained on the QM-9 dataset. Each experiment is conducted using 3 random seeds, and the mean value is presented. The best average result is highlighted in bold.}
    \label{tab:qm9}
    \vspace{-10pt}
\end{table*}

\textbf{Metrics.}
We have two types of metrics to measure MOO algorithms.
{\em I. Model performance metrics.} We consider two widely-used metrics to represent the overall performance of one MOO method $m$: \textbf{(1)} \dm{}, the average per-task performance drop of a method $m$ relative to the STL baseline denoted as $B$:
$\bm{\Delta}\bm{m}\% = \frac{1}{N}\sum_{n=1}^N (-1)^{\delta_n} \frac{(M_{m,n} - M_{B,n})}{M_{B,n}} \times 100,$
where $M_{B,n}$ and $M_{m,n}$ represent the STL and $m$'s value for metric $M_n$, respectively. Here, $\delta_n$ equals 1 if a higher value of $M_n$ is better or 0 if a lower value of $M_n$ is better. \textbf{(2) Mean Rank (MR)}, the average rank of each method in the tasks. For example, if a method ranks first in every task (whether higher or lower is better), \textbf{MR} will be 1. Note that in practice, lower values of \dm{} and \textbf{MR} indicate better overall performance in the final result.
{\em II. 
Convergence metrics.} 
We calculate the training time per epoch and total training time to evaluate the convergence performance of MOO algorithms. Based on various datasets, we also calculate the test loss for a visualization comparison through the training process.

\subsection{Regression: QM-9}

QM-9 dataset~\citep{blum2009970} is a crucial benchmark widely used in the field of graph neural network learning, especially in chemical informatics applications. 
With over 13K molecules, each molecule is intricately depicted as a graph, with nodes representing atoms and edges symbolizing the chemical bonds between them. These graphs feature detailed node and edge characteristics such as atomic number, atom type, bond type, and bond order. 
% Following the experimental setup of NASHMTL and FAMO~\citep{navon2022multi, liu2024famo}, 
The goal is to predict 11 molecule properties. Utilizing 110K molecules from QM9 in PyTorch Geometric~\citep{fey2019fast} for training, 10K for validation, and the remaining 10K for testing. 
% The dataset's characteristic is the varying scales of its 11 properties, leading to significant task imbalances in MOO.

\textbf{Performance.}
We assess two common metrics, \textbf{MR} and \dm{}, for various MOO algorithms to ensure effective evaluations.
Table~\ref{tab:qm9} shows that \psmg{} achieves comparable or even better performance in both \textbf{MR} and \dm{} when compared with existing baselines.
Specifically, it has the best performance in 3 out of the 11 total tasks, namely $\alpha$, $\epsilon_\text{LUMO}$, and $c_v$.
When we take a closer look at the training process, we observe that \psmg{} has a faster and more smooth process as illustrated in Figure~\ref{fig:QM9loss}. 
Specifically, \psmg{} demonstrates a smoother descent towards lower test loss at least in the early stages.
In addition, \psmg{} exhibits resistance to overfitting, distinguishing itself from IMTL-G and NashMTL due to its straightforward yet efficient structure.
% The full results can be found in Appendix~\ref{sec:sup_ex}.
% These findings indicate that \psmg{} is an effective approach due to its strong convergence and generalization capabilities.

% In the training process of QM-9, as shown in Figure~\ref{fig:QM9loss}, \psmg{} demonstrates a smoother descent towards lower test loss in the early stages, such as the first 100 epochs, compared to other methods. Particularly, \psmg{} exhibits resistance to overfitting, distinguishing itself from IMTL-G and NashMTL due to its straightforward yet efficient structure. 

\textbf{Convergence measurements.}
In Table~\ref{tab:qm9_time}, we show the training time per epoch and the total training time required to achieve a certain loss. In Figure~\ref{fig:QM9loss}, we use the average Mean Absolute Error (MAE) across all molecular property prediction tasks as test loss, calculated after scaling both the predictions and the ground truth values by the standard deviation. \psmg{} ranks third in training speed per epoch but achieves the average test loss (<100, <75) in the shortest overall time. The visualization results can be found in Appendix~\ref{sec:sup_pre}.

\begin{table}[htb]
    \centering
    \resizebox{\columnwidth}{!}{%
    \begin{tabular}{lccr}
    \toprule
    \multirow{2}{*}{\textbf{Method}} & Training time $\downarrow$ & Avg loss $\downarrow$ & Avg loss $\downarrow$\\
    & (mean per epoch) & < 100 & < 75 \\ 
    \midrule
    % \ls{}   & 101.65 \fs{0.75} & 0.49\\
    % \mgda{}             & 776.50 \fs{17.89} & 0.49\\
    % % \midrule
    % \pcgrad{}                  & 369.02 \fs{0.92} & 0.49\\
    % \cagrad{} & 317.31 \fs{1.29} & 0.49\\
    % \imtlg{} & 322.58 \fs{0.86} & 0.49\\
    % \famo{} & 117.84 \fs{3.17} & 0.49\\
    % \nashmtl{}        & 512.61 \fs{3.09} & 0.49\\
    % \fairgrad{}        & 323.33 \fs{0.94} & 0.49\\
    % \psmg{}        & 133.64 \fs{1.82} & 0.49\\
    % \ls{}   & \best{1.70} & 75.68 & 170.28 & \textbackslash{}\\
    % \mgda{}             & 12.94 & 436.92 & 754.68 & \textbackslash{}\\
    % % \midrule
    % \pcgrad{}                  & 6.15 & 259.56 & 630.36 & \textbackslash{}\\
    % \cagrad{} & 5.29 & 233.64 & 451.35 & \textbackslash{}\\
    % \imtlg{} & 5.38 & 275.40 & 604.80 & 1398.60\\
    % \famo{} & 1.96 & 72.52 & 131.32 & \best{439.04}\\
    % \nashmtl{}        & 8.45 & 287.30 & 371.80 & 870.35\\
    % \fairgrad{}        & 5.40 & 162.60 & 314.36 & 1620.58\\
    % \psmg{}        & 2.23 & \best{69.13} & \best{109.27} & 653.39\\
    \ls{}   & \best{1.70} & 75.68 & 170.28\\
    \mgda{}             & 12.94 & 436.92 & 754.68\\
    % \midrule
    \pcgrad{}                  & 6.15 & 259.56 & 630.36\\
    \cagrad{} & 5.29 & 233.64 & 451.35\\
    \imtlg{} & 5.38 & 275.40 & 604.80\\
    \famo{} & 1.96 & 72.52 & 131.32\\
    \nashmtl{}        & 8.45 & 287.30 & 371.80\\
    \fairgrad{}        & 5.40 & 162.60 & 314.36\\
    \midrule
    \psmg{}        & 2.23 & \best{69.13} & \best{109.27}\\
    \bottomrule
    \end{tabular}
    }
    \vspace{2pt}
    \caption{Results obtained on the QM-9 dataset. Each experiment is conducted using 3 random seeds, and the mean value is presented. The best average result is highlighted in bold.}
    \label{tab:qm9_time}
    \vspace{-10pt}
\end{table}

\begin{table*}[htb]
    \centering
    \resizebox{\textwidth}{!}{%
    \begin{tabular}{lrrrrrrrrrrr}
    \toprule
      &  \multicolumn{2}{c}{Segmentation} & \multicolumn{2}{c}{Depth} & \multicolumn{5}{c}{Surface Normal} & &\\
    \cmidrule(lr){2-3}\cmidrule(lr){4-5}\cmidrule(lr){6-10}
    \textbf{Method} &  \multirow{2}{*}{mIoU $\uparrow$} & \multirow{2}{*}{Pix Acc $\uparrow$} & \multirow{2}{*}{Abs Err $\downarrow$} & \multirow{2}{*}{Rel Err $\downarrow$} & \multicolumn{2}{c}{Angle Dist $\downarrow$} & \multicolumn{3}{c}{Within $t^\circ$ $\uparrow$}  & \mr{} $\downarrow$ &  \dm{} $\downarrow$ \\
    \cmidrule(lr){6-7}\cmidrule(lr){8-10}
    & & & & & Mean & Median & 11.25 & 22.5 & 30  &\\
    \midrule 
    \stl{}       & 38.30 & 63.76 & 0.6754 & 0.2780 & 25.01 & 19.21 & 30.14 & 57.20 & 69.15   &  &     \\
    \midrule
    \ls{}        & 39.29 & 65.33 & 0.5493 & 0.2263 & 28.15 & 23.96 & 22.09 & 47.50 & 61.08   & 11.44 & 5.59  \\
    \si{}        & 38.45 & 64.27 & 0.5354 & 0.2201 & 27.60 & 23.37 & 22.53 & 48.57 & 62.32   & 10.11 & 4.39  \\
    \rlw{}       & 37.17 & 63.77 & 0.5759 & 0.2410 & 28.27 & 24.18 & 22.26 & 47.05 & 60.62   & 14.11 & 7.78  \\
    \dwa{}       & 39.11 & 65.31 & 0.5510 & 0.2285 & 27.61 & 23.18 & 24.17 & 50.18 & 62.39   & 10.44 & 3.57  \\
    \uw{}        & 36.87 & 63.17 & 0.5446 & 0.2260 & 27.04 & 22.61 & 23.54 & 49.05 & 63.65   & 10.11 & 4.05  \\
    \mgda{}      & 30.47 & 59.90 & 0.6070 & 0.2555 & 24.88 & 19.45 & 29.18 & 56.88 & 69.36   & 8.11 & 1.38  \\
    \pcgrad{}    & 38.06 & 64.64 & 0.5550 & 0.2325 & 27.41 & 22.80 & 23.86 & 49.83 & 63.14   & 10.67 & 3.97  \\
    \graddrop{}  & 39.39 & 65.12 & 0.5455 & 0.2279 & 27.48 & 22.96 & 23.38 & 49.44 & 62.87   & 9.56 & 3.58  \\
    \cagrad{}    & 39.79 & 65.49 & 0.5486 & 0.2250 & 26.31 & 21.58 & 25.61 & 52.36 & 65.58   & 7.00 & 0.20  \\
    \imtlg{}     & 39.35 & 65.60 & 0.5426 & 0.2256 & 26.02 & 21.19 & 26.20 & 53.13 & 66.24   & 6.33 & -0.76  \\
    \nashmtl{}   & 40.13 & 65.93 & 0.5261 & 0.2171 & 25.26 & 20.08 & 28.40 & 55.47 & 68.15   & 4.22 & -4.04  \\
    \famo{}      & 38.88 & 64.90 & 0.5474 & 0.2194 & 25.06 & 19.57 & 29.21 & 56.61 & 68.98 & 5.11 & -4.10 \\
    \fairgrad{}      & 39.74 & \best{66.01} & 0.5377 & 0.2236 & 24.84 & 19.60 & 29.26 & 56.58 & 69.16 & \best{3.22} & -4.66 \\
    \sdmgrad{}      & \best{40.47} & 65.90 & \best{0.5225} & \best{0.2084} & 25.07 & 19.99 & 28.54 & 55.74 & 68.53 & 3.44 & \best{-4.84}\\
    \midrule
    \psmg{}      & 35.44 & 63.78 & 0.5494 & 0.2369 & \best{24.83} & \best{18.89} & \best{30.68} & \best{58.00} & \best{69.84} & 6.11 & -3.62 \\
    \bottomrule  
    \end{tabular}
    }
    \vspace{2pt}
    \caption{Results obtained on NYU-v2. Each experiment is conducted using 3 random seeds, and the mean value is presented. The best average result is highlighted in bold.}
    \label{tab:nyu-v2}
    % \vspace{-10pt}
\end{table*}

\begin{table*}[htb]
    \centering
    \resizebox{\textwidth}{!}{
    \begin{tabular}{lccccccc}
    \toprule
    \multirow{2}{*}{\textbf{Method}} & Training time $\downarrow$ & Semantic loss $\downarrow$ & Semantic loss $\downarrow$ & Depth loss $\downarrow$ & Depth loss $\downarrow$ & Normal loss $\downarrow$ & Normal loss $\downarrow$\\
    & (mean per epoch) & < 1.40 & < 1.20 & < 0.65 & < 0.55 & < 0.20 & < 0.16\\ 
    \midrule
    \ls{}   & \best{1.36} & 32.14 & 66.28 & 46.18 & 159.36 & \best{30.84} & 152.76\\
    \mgda{}          & 3.34 & 166.43 & 384.32 & 218.88 & 590.96 & 56.44 & 249.77\\
    % \midrule
    \pcgrad{}    & 3.31 & 72.38 & 154.63 & 98.79 & 299.39 & 75.67 & 332.29\\
    \cagrad{} & 3.53 & 63.18 & 137.89 & 94.77 & 351.06 & 66.69 & 213.11\\
    \imtlg{} & 4.63 & 101.42 & 239.72 & 147.52 & 470.22 & 87.59 & 318.09\\
    \famo{} & 1.65 & 32.69 & 68.46 & 73.35 & \best{156.26} & 34.23 & 114.13\\
    \nashmtl{} & 3.49 & 76.34 & 159.62 & 79.81 & 277.60 & 65.93 & 201.26\\
    \fairgrad{}  & 2.93 & 141.82 & 328.86 & 148.47 & 334.53 & 48.20 & 177.41\\
    \sdmgrad{}   & 1.98 & 31.52 & 72.89 & 63.04 & 200.94 & 35.46 & 128.05\\
    \midrule
    \psmg{}      & 1.85 & \best{29.28} & \best{64.54} & \best{43.07} & 176.66 & 31.11 & \best{113.46}\\
    \bottomrule
    \end{tabular}
    }
    % \vspace{10pt}
    % \label{tab:qm9time}
  \captionof{table}{Training time per epoch [Min.] and convergence process (averaged over 3 random seeds) in all tasks on NYU-v2.}
  \label{tab:nyuv2time}
  \vspace{-10pt}
\end{table*}

\subsection{Dense Prediction: NYU-v2}

NYU-v2~\citep{silberman2012indoor} is an indoor scene dataset containing 1449 RGBD images, each annotated with dense per-pixel labeling across 13 distinct classes. 
% These classes categorize elements commonly found indoors, like walls, furniture, and floors, facilitating detailed scene analysis. 
The dataset supports 3 tasks for MOO experiments: image segmentation, depth prediction, and surface normal prediction. 
% Image segmentation divides an image into segments for easier analysis, crucial for identifying and classifying objects within a scene. Depth prediction estimates object distances using only RGB information. Surface normal prediction determines surface orientation, vital for accurate graphics rendering and improved spatial understanding in robotics.

\textbf{Performance.}
As shown in Table~\ref{tab:nyu-v2}, \psmg{} achieves better performance in both \textbf{MR} and \dm{} compared to existing MOO algorithms. It performs the best in 5 out of the 9 total tasks.
In Figure~\ref{fig:nyuv2_loss}, PSMGD achieves consistently low test losses across all tasks compared to other methods in the initial stage, indicating superior performance across all 3 tasks throughout the 200 training epochs.
% Its \dm{} is close to them, even though it has some lower-ranked task performance that affects its overall score in \textbf{MR}. 
% Furthermore, 

% These findings indicate that \psmg{} is an effective approach due to its strong convergence and generalization capabilities.

\textbf{Convergence measurements.}
% In the training process, our proposed \psmg{} demonstrates significant advancements in segmentation, depth estimation, and surface normal prediction of the NYU-v2 dataset, as evidenced in Figure~\ref{fig:nyuv2_loss} and Table~\ref{tab:nyuv2time}.
In Figure~\ref{fig:nyuv2_loss}, PSMGD outperforms other MOO methods in reducing test loss for all three tasks on the NYU-v2 dataset, particularly in the early training stages. 
Table~\ref{tab:nyuv2time} further highlights \psmg{}'s efficiency with its third lowest training time per epoch (1.85), significantly outperforming other classical gradient manipulation methods like \imtlg{}~\citep{liu2021towards} and \nashmtl{}~\citep{navon2022multi}.
PSMGD achieves the fastest training times by reaching semantic loss in 29.28 and 64.54 minutes, depth loss in 43.07 minutes, and normal loss in 113.46 minutes, when considering the following thresholds: <1.40 and <1.20 for semantic loss, <0.65 for depth loss, and <0.16 for normal loss.
The visualization results can be found in Appendix~\ref{sec:sup_pre}.
\section{Conclusion} \label{sec: conclusion}

In this paper, we propose a novel and efficient algorithm, Periodic Stochastic Multi-Gradient Descent (PSMGD), to accelerate MOO. Our PSMGD algorithm periodically computes dynamic weights and reuses them, significantly reducing the computational load and speeding up MOO training. We establish that PSMGD achieves state-of-the-art convergence rates for strongly convex, general convex, and non-convex functions. Moreover, we demonstrate the superior backpropagation (BP) complexity of our PSMGD algorithm. Extensive experiments confirm that PSMGD delivers performance comparable to or better than existing MOO algorithms, with a substantial reduction in training time. We believe future work could explore preference-based solutions or the entire Pareto set using our efficient PSMGD algorithm.

\section*{Acknowledgments}
JL acknowledges the funding from NSF grants CAREER CNS-2110259, CNS-2112471, IIS-2324052, DARPA YFA D24AP00265, ONR grant N00014-24-1-2729, and AFRL grant PGSC-SC-111374-19s. HY acknowledges the funding support from AI Seed Funding and GWBC Award at RIT.
% \newpage
% \bibliographystyle{apalike}
\bibliography{references}

\begin{thebibliography}{44}
\providecommand{\natexlab}[1]{#1}

\bibitem[{Badrinarayanan, Kendall, and Cipolla(2017)}]{badrinarayanan2017segnet}
Badrinarayanan, V.; Kendall, A.; and Cipolla, R. 2017.
\newblock Segnet: A deep convolutional encoder-decoder architecture for image segmentation.
\newblock \emph{IEEE transactions on pattern analysis and machine intelligence}, 39(12): 2481--2495.

\bibitem[{Ban and Ji(2024)}]{ban2024fair}
Ban, H.; and Ji, K. 2024.
\newblock Fair Resource Allocation in Multi-Task Learning.
\newblock \emph{arXiv preprint arXiv:2402.15638}.

\bibitem[{Belakaria et~al.(2020)Belakaria, Deshwal, Jayakodi, and Doppa}]{belakaria2020uncertainty}
Belakaria, S.; Deshwal, A.; Jayakodi, N.~K.; and Doppa, J.~R. 2020.
\newblock Uncertainty-aware search framework for multi-objective Bayesian optimization.
\newblock In \emph{Proceedings of the AAAI Conference on Artificial Intelligence}, volume~34, 10044--10052.

\bibitem[{Blum and Reymond(2009)}]{blum2009970}
Blum, L.~C.; and Reymond, J.-L. 2009.
\newblock 970 million druglike small molecules for virtual screening in the chemical universe database GDB-13.
\newblock \emph{Journal of the American Chemical Society}, 131(25): 8732--8733.

\bibitem[{Caruana(1997)}]{caruana1997multitask}
Caruana, R. 1997.
\newblock Multitask learning.
\newblock \emph{Machine learning}, 28: 41--75.

\bibitem[{Chen et~al.(2024)Chen, Fernando, Ying, and Chen}]{chen2024three}
Chen, L.; Fernando, H.; Ying, Y.; and Chen, T. 2024.
\newblock Three-way trade-off in multi-objective learning: Optimization, generalization and conflict-avoidance.
\newblock \emph{Advances in Neural Information Processing Systems}, 36.

\bibitem[{Chen et~al.(2018)Chen, Badrinarayanan, Lee, and Rabinovich}]{chen2018gradnorm}
Chen, Z.; Badrinarayanan, V.; Lee, C.-Y.; and Rabinovich, A. 2018.
\newblock Gradnorm: Gradient normalization for adaptive loss balancing in deep multitask networks.
\newblock In \emph{International conference on machine learning}, 794--803. PMLR.

\bibitem[{Chen et~al.(2020)Chen, Ngiam, Huang, Luong, Kretzschmar, Chai, and Anguelov}]{chen2020just}
Chen, Z.; Ngiam, J.; Huang, Y.; Luong, T.; Kretzschmar, H.; Chai, Y.; and Anguelov, D. 2020.
\newblock Just pick a sign: Optimizing deep multitask models with gradient sign dropout.
\newblock \emph{Advances in Neural Information Processing Systems}, 33: 2039--2050.

\bibitem[{Cordts et~al.(2016)Cordts, Omran, Ramos, Rehfeld, Enzweiler, Benenson, Franke, Roth, and Schiele}]{cordts2016cityscapes}
Cordts, M.; Omran, M.; Ramos, S.; Rehfeld, T.; Enzweiler, M.; Benenson, R.; Franke, U.; Roth, S.; and Schiele, B. 2016.
\newblock The cityscapes dataset for semantic urban scene understanding.
\newblock In \emph{Proceedings of the IEEE conference on computer vision and pattern recognition}, 3213--3223.

\bibitem[{Deb et~al.(2002)Deb, Pratap, Agarwal, and Meyarivan}]{deb2002fast}
Deb, K.; Pratap, A.; Agarwal, S.; and Meyarivan, T. 2002.
\newblock A fast and elitist multiobjective genetic algorithm: NSGA-II.
\newblock \emph{IEEE transactions on evolutionary computation}, 6(2): 182--197.

\bibitem[{D{\'e}sid{\'e}ri(2012)}]{desideri2012multiple}
D{\'e}sid{\'e}ri, J.-A. 2012.
\newblock Multiple-gradient descent algorithm (MGDA) for multiobjective optimization.
\newblock \emph{Comptes Rendus Mathematique}, 350(5-6): 313--318.

\bibitem[{Fernando et~al.(2022)Fernando, Shen, Liu, Chaudhury, Murugesan, and Chen}]{fernando2022mitigating}
Fernando, H.~D.; Shen, H.; Liu, M.; Chaudhury, S.; Murugesan, K.; and Chen, T. 2022.
\newblock Mitigating gradient bias in multi-objective learning: A provably convergent approach.
\newblock In \emph{The Eleventh International Conference on Learning Representations}.

\bibitem[{Fey and Lenssen(2019)}]{fey2019fast}
Fey, M.; and Lenssen, J.~E. 2019.
\newblock Fast graph representation learning with PyTorch Geometric.
\newblock \emph{arXiv preprint arXiv:1903.02428}.

\bibitem[{Fliege and Svaiter(2000)}]{fliege2000steepest}
Fliege, J.; and Svaiter, B.~F. 2000.
\newblock Steepest descent methods for multicriteria optimization.
\newblock \emph{Mathematical methods of operations research}, 51: 479--494.

\bibitem[{Fliege, Vaz, and Vicente(2019)}]{fliege2019complexity}
Fliege, J.; Vaz, A. I.~F.; and Vicente, L.~N. 2019.
\newblock Complexity of gradient descent for multiobjective optimization.
\newblock \emph{Optimization Methods and Software}, 34(5): 949--959.

\bibitem[{Fonseca and Fleming(1996)}]{fonseca1996performance}
Fonseca, C.~M.; and Fleming, P.~J. 1996.
\newblock On the performance assessment and comparison of stochastic multiobjective optimizers.
\newblock In \emph{International conference on parallel problem solving from nature}, 584--593. Springer.

\bibitem[{Hayes et~al.(2022)Hayes, R{\u{a}}dulescu, Bargiacchi, K{\"a}llstr{\"o}m, Macfarlane, Reymond, Verstraeten, Zintgraf, Dazeley, Heintz et~al.}]{hayes2022practical}
Hayes, C.~F.; R{\u{a}}dulescu, R.; Bargiacchi, E.; K{\"a}llstr{\"o}m, J.; Macfarlane, M.; Reymond, M.; Verstraeten, T.; Zintgraf, L.~M.; Dazeley, R.; Heintz, F.; et~al. 2022.
\newblock A practical guide to multi-objective reinforcement learning and planning.
\newblock \emph{Autonomous Agents and Multi-Agent Systems}, 36(1): 26.

\bibitem[{Javaloy and Valera(2021)}]{javaloy2021rotograd}
Javaloy, A.; and Valera, I. 2021.
\newblock Rotograd: Gradient homogenization in multitask learning.
\newblock \emph{arXiv preprint arXiv:2103.02631}.

\bibitem[{Kendall, Gal, and Cipolla(2018)}]{kendall2018multi}
Kendall, A.; Gal, Y.; and Cipolla, R. 2018.
\newblock Multi-task learning using uncertainty to weigh losses for scene geometry and semantics.
\newblock In \emph{Proceedings of the IEEE conference on computer vision and pattern recognition}, 7482--7491.

\bibitem[{Kingma and Ba(2014)}]{kingma2014adam}
Kingma, D.~P.; and Ba, J. 2014.
\newblock Adam: A method for stochastic optimization.
\newblock \emph{arXiv preprint arXiv:1412.6980}.

\bibitem[{Kurin et~al.(2022)Kurin, De~Palma, Kostrikov, Whiteson, and Mudigonda}]{kurin2022defense}
Kurin, V.; De~Palma, A.; Kostrikov, I.; Whiteson, S.; and Mudigonda, P.~K. 2022.
\newblock In defense of the unitary scalarization for deep multi-task learning.
\newblock \emph{Advances in Neural Information Processing Systems}, 35: 12169--12183.

\bibitem[{Laumanns and Ocenasek(2002)}]{laumanns2002bayesian}
Laumanns, M.; and Ocenasek, J. 2002.
\newblock Bayesian optimization algorithms for multi-objective optimization.
\newblock In \emph{International Conference on Parallel Problem Solving from Nature}, 298--307. Springer.

\bibitem[{LeCun et~al.(1998)LeCun, Bottou, Bengio, and Haffner}]{lecun1998gradient}
LeCun, Y.; Bottou, L.; Bengio, Y.; and Haffner, P. 1998.
\newblock Gradient-based learning applied to document recognition.
\newblock \emph{Proceedings of the IEEE}, 86(11): 2278--2324.

\bibitem[{Lin, Feiyang, and Zhang(2021)}]{lin2021closer}
Lin, B.; Feiyang, Y.; and Zhang, Y. 2021.
\newblock A closer look at loss weighting in multi-task learning.

\bibitem[{Liu et~al.(2024)Liu, Feng, Stone, and Liu}]{liu2024famo}
Liu, B.; Feng, Y.; Stone, P.; and Liu, Q. 2024.
\newblock Famo: Fast adaptive multitask optimization.
\newblock \emph{Advances in Neural Information Processing Systems}, 36.

\bibitem[{Liu et~al.(2021{\natexlab{a}})Liu, Liu, Jin, Stone, and Liu}]{liu2021conflict}
Liu, B.; Liu, X.; Jin, X.; Stone, P.; and Liu, Q. 2021{\natexlab{a}}.
\newblock Conflict-averse gradient descent for multi-task learning.
\newblock \emph{Advances in Neural Information Processing Systems}, 34: 18878--18890.

\bibitem[{Liu et~al.(2021{\natexlab{b}})Liu, Li, Kuang, Xue, Chen, Yang, Liao, and Zhang}]{liu2021towards}
Liu, L.; Li, Y.; Kuang, Z.; Xue, J.; Chen, Y.; Yang, W.; Liao, Q.; and Zhang, W. 2021{\natexlab{b}}.
\newblock Towards impartial multi-task learning.
\newblock iclr.

\bibitem[{Liu, Johns, and Davison(2019)}]{liu2019end}
Liu, S.; Johns, E.; and Davison, A.~J. 2019.
\newblock End-to-end multi-task learning with attention.
\newblock In \emph{Proceedings of the IEEE/CVF conference on computer vision and pattern recognition}, 1871--1880.

\bibitem[{Liu and Vicente(2021)}]{liu2021stochastic}
Liu, S.; and Vicente, L.~N. 2021.
\newblock The stochastic multi-gradient algorithm for multi-objective optimization and its application to supervised machine learning.
\newblock \emph{Annals of Operations Research}, 1--30.

\bibitem[{Liu et~al.(2015)Liu, Luo, Wang, and Tang}]{liu2015deep}
Liu, Z.; Luo, P.; Wang, X.; and Tang, X. 2015.
\newblock Deep learning face attributes in the wild.
\newblock In \emph{Proceedings of the IEEE international conference on computer vision}, 3730--3738.

\bibitem[{Mahapatra et~al.(2023)Mahapatra, Dong, Chen, and Momma}]{mahapatra2023multi}
Mahapatra, D.; Dong, C.; Chen, Y.; and Momma, M. 2023.
\newblock Multi-label learning to rank through multi-objective optimization.
\newblock In \emph{Proceedings of the 29th ACM SIGKDD Conference on Knowledge Discovery and Data Mining}, 4605--4616.

\bibitem[{Maheshwari and Perrot(2022)}]{maheshwari2022fairgrad}
Maheshwari, G.; and Perrot, M. 2022.
\newblock Fairgrad: Fairness aware gradient descent.
\newblock \emph{arXiv preprint arXiv:2206.10923}.

\bibitem[{Miettinen(1999)}]{miettinen1999nonlinear}
Miettinen, K. 1999.
\newblock \emph{Nonlinear multiobjective optimization}, volume~12.
\newblock Springer Science \& Business Media.

\bibitem[{Navon et~al.(2022)Navon, Shamsian, Achituve, Maron, Kawaguchi, Chechik, and Fetaya}]{navon2022multi}
Navon, A.; Shamsian, A.; Achituve, I.; Maron, H.; Kawaguchi, K.; Chechik, G.; and Fetaya, E. 2022.
\newblock Multi-task learning as a bargaining game.
\newblock \emph{arXiv preprint arXiv:2202.01017}.

\bibitem[{Pardalos et~al.(2017)Pardalos, {\v{Z}}ilinskas, {\v{Z}}ilinskas et~al.}]{pardalos2017non}
Pardalos, P.~M.; {\v{Z}}ilinskas, A.; {\v{Z}}ilinskas, J.; et~al. 2017.
\newblock \emph{Non-convex multi-objective optimization}.
\newblock Springer.

\bibitem[{Sener and Koltun(2018)}]{sener2018multi}
Sener, O.; and Koltun, V. 2018.
\newblock Multi-task learning as multi-objective optimization.
\newblock \emph{Advances in neural information processing systems}, 31.

\bibitem[{Silberman et~al.(2012)Silberman, Hoiem, Kohli, and Fergus}]{silberman2012indoor}
Silberman, N.; Hoiem, D.; Kohli, P.; and Fergus, R. 2012.
\newblock Indoor segmentation and support inference from rgbd images.
\newblock In \emph{Computer Vision--ECCV 2012: 12th European Conference on Computer Vision, Florence, Italy, October 7-13, 2012, Proceedings, Part V 12}, 746--760. Springer.

\bibitem[{Tanabe, Fukuda, and Yamashita(2019)}]{tanabe2019proximal}
Tanabe, H.; Fukuda, E.~H.; and Yamashita, N. 2019.
\newblock Proximal gradient methods for multiobjective optimization and their applications.
\newblock \emph{Computational Optimization and Applications}, 72: 339--361.

\bibitem[{Xiao, Ban, and Ji(2024)}]{xiao2024direction}
Xiao, P.; Ban, H.; and Ji, K. 2024.
\newblock Direction-oriented multi-objective learning: Simple and provable stochastic algorithms.
\newblock \emph{Advances in Neural Information Processing Systems}, 36.

\bibitem[{Xin et~al.(2022)Xin, Ghorbani, Gilmer, Garg, and Firat}]{xin2022current}
Xin, D.; Ghorbani, B.; Gilmer, J.; Garg, A.; and Firat, O. 2022.
\newblock Do current multi-task optimization methods in deep learning even help?
\newblock \emph{Advances in neural information processing systems}, 35: 13597--13609.

\bibitem[{Yang et~al.(2024)Yang, Liu, Liu, Dong, and Momma}]{yang2024federated}
Yang, H.; Liu, Z.; Liu, J.; Dong, C.; and Momma, M. 2024.
\newblock Federated multi-objective learning.
\newblock \emph{Advances in Neural Information Processing Systems}, 36.

\bibitem[{Yu et~al.(2020)Yu, Kumar, Gupta, Levine, Hausman, and Finn}]{yu2020gradient}
Yu, T.; Kumar, S.; Gupta, A.; Levine, S.; Hausman, K.; and Finn, C. 2020.
\newblock Gradient surgery for multi-task learning.
\newblock \emph{Advances in Neural Information Processing Systems}, 33: 5824--5836.

\bibitem[{Zhang and Li(2007)}]{zhang2007moea}
Zhang, Q.; and Li, H. 2007.
\newblock MOEA/D: A multiobjective evolutionary algorithm based on decomposition.
\newblock \emph{IEEE Transactions on evolutionary computation}, 11(6): 712--731.

\bibitem[{Zhou et~al.(2022)Zhou, Zhang, Jiang, Zhong, Gu, and Zhu}]{zhou2022convergence}
Zhou, S.; Zhang, W.; Jiang, J.; Zhong, W.; Gu, J.; and Zhu, W. 2022.
\newblock On the convergence of stochastic multi-objective gradient manipulation and beyond.
\newblock \emph{Advances in Neural Information Processing Systems}, 35: 38103--38115.

\end{thebibliography}

\newpage
\clearpage
\section{Reproducibility Checklist} \label{sec:checklist}

\paragraph{This paper:}
\begin{itemize}
    \item Includes a conceptual outline and/or pseudocode description of AI methods introduced (yes/partial/no/NA) {\bf yes}
    \item Clearly delineates statements that are opinions, hypothesis, and speculation from objective facts and results (yes/no) {\bf yes}
    \item Provides well-marked pedagogical references for less-familiar readers to gain background necessary to replicate the paper (yes/no) {\bf yes}
\end{itemize}

\paragraph{Does this paper make theoretical contributions? (yes/no) {\bf yes}}
\begin{itemize}
    \item All assumptions and restrictions are stated clearly and formally. (yes/partial/no) {\bf yes}
    \item All novel claims are stated formally (e.g., in theorem statements). (yes/partial/no) {\bf yes}
    \item Proofs of all novel claims are included. (yes/partial/no) {\bf yes}
    \item Proof sketches or intuitions are given for complex and/or novel results. (yes/partial/no) {\bf yes}
    \item Appropriate citations to theoretical tools used are given. (yes/partial/no) {\bf yes}
    \item All theoretical claims are demonstrated empirically to hold. (yes/partial/no/NA) {\bf yes}
    \item All experimental code used to eliminate or disprove claims is included. (yes/no/NA) {\bf yes}
\end{itemize}

\paragraph{Does this paper rely on one or more datasets? (yes/no) {\bf yes}}
\begin{itemize}
    \item A motivation is given for why the experiments are conducted on the selected datasets (yes/partial/no/NA) {\bf yes}
    \item All novel datasets introduced in this paper are included in a data appendix. (yes/partial/no/NA) {\bf yes}
    \item All novel datasets introduced in this paper will be made publicly available upon publication of the paper with a license that allows free usage for research purposes. (yes/partial/no/NA) {\bf yes}
    \item All datasets drawn from the existing literature (potentially including authors’ own previously published work) are accompanied by appropriate citations. (yes/no/NA) {\bf yes}
    \item All datasets drawn from the existing literature (potentially including authors’ own previously published work) are publicly available. (yes/partial/no/NA) {\bf yes}
    \item All datasets that are not publicly available are described in detail, with explanation why publicly available alternatives are not scientifically satisficing. (yes/partial/no/NA) {\bf NA}
\end{itemize}

\paragraph{Does this paper include computational experiments? (yes/no)}
\begin{itemize}
    \item Any code required for pre-processing data is included in the appendix. (yes/partial/no) {\bf yes}
    \item All source code required for conducting and analyzing the experiments is included in a code appendix. (yes/partial/no) {\bf yes}
    \item All source code required for conducting and analyzing the experiments will be made publicly available upon publication of the paper with a license that allows free usage for research purposes. (yes/partial/no) {\bf yes}
    \item All source code implementing new methods have comments detailing the implementation, with references to the paper where each step comes from (yes/partial/no) {\bf yes}
    \item If an algorithm depends on randomness, then the method used for setting seeds is described in a way sufficient to allow replication of results. (yes/partial/no/NA) {\bf yes}
    \item This paper specifies the computing infrastructure used for running experiments (hardware and software), including GPU/CPU models; amount of memory; operating system; names and versions of relevant software libraries and frameworks. (yes/partial/no) {\bf yes}
    \item This paper formally describes evaluation metrics used and explains the motivation for choosing these metrics. (yes/partial/no) {\bf yes}
    \item This paper states the number of algorithm runs used to compute each reported result. (yes/no) {\bf yes}
    \item Analysis of experiments goes beyond single-dimensional summaries of performance (e.g., average; median) to include measures of variation, confidence, or other distributional information. (yes/no) {\bf yes}
    \item The significance of any improvement or decrease in performance is judged using appropriate statistical tests (e.g., Wilcoxon signed-rank). (yes/partial/no) {\bf yes}
    \item This paper lists all final (hyper-)parameters used for each model/algorithm in the paper’s experiments. (yes/partial/no/NA) {\bf yes}
    \item This paper states the number and range of values tried per (hyper-)parameter during development of the paper, along with the criterion used for selecting the final parameter setting. (yes/partial/no/NA) {\bf yes}
\end{itemize}

% \fi

\newpage
\appendix
% !TEX root = main.tex

\begin{center}
\textbf{Appendix}
\end{center}

\section{Proof} \label{sec: appdx_theory}
\allowdisplaybreaks

\subsection{Convergence Analysis} \label{sub:CA}

\textbf{Notation.}
For notation clarity, we have the following definition:
$\nabla \G(\x_t, \Blambda_t) = \sum_{s \in [S]} \lambda_{t, s} \nabla f_s(\x_t)$, $\hat{ \nabla} \G(\x_t, \Blambda_t) = \sum_{s \in [S]} \lambda_{t, s} \nabla f_s(\x_t, \xi_t)$.
$\hat{\Blambda}_t^*$ is calculated by QP using stochastic gradient $\nabla f_s(\x_t, \xi_t)$, 
and $\Blambda_t = \alpha_t \Blambda_{t-R} + (1 - \alpha_t) \hat{\Blambda}_t^*.$

\textbf{Convergence Metrics.}
We note that existing works use similar but slightly different metrics to measure the convergence of MOO. 
For strongly-convex functions, \cite{liu2021stochastic} uses $ \min_{t=1, \dots, T} \sum_{s \in [S]} \left[ \hat{\lambda}_{t, s}^{*}  f_s(\x_t) - \bar{\lambda}_T f_s(\x_*) \right]$ where $\bar{\lambda}_T = \sum_{t=1}^T \frac{t}{\sum_{t=1}^T t} \hat{\lambda_t}$. 
Here $\hat{\lambda}_{t, s}^{*}$ is calculated by the quadratic programming problem~\ref{eq: QP} with stochastic gradients.
\cite{yang2024federated} uses $\Delta_Q^t = \sum_{s \in [S]} \lambda_{t, s}^{*} \left[ f_s(\x_t) - f_s(\x_*) \right]$ as the metrics. 
For non-convex functions, $\| \bd_t \|^2 = \| \boldsymbol{\lambda}_t^{*T} \nabla \F(\x_t) \|^2$ (or equivalently $\mathbb{E} [\min_{\Blambda_t} \| \nabla \G(\x_t, \Blambda_t) \|^2]$) is usually used as the metrics.
Our convergence metrics follow the conditions from ~\cite{zhou2022convergence}.
For non-convex functions, we use $\frac{1}{T} \sum_{t \in [T]} \mathbb{E} [\| \nabla \G(\x_t, \Blambda_t) \|^2]$ as the metrics, which can directly connect with metrics used in other works as $\frac{1}{T} \sum_{t \in [T]} \mathbb{E} [\min_{\Blambda_t} \| \nabla \G(\x_t, \Blambda_t) \|^2] \leq \frac{1}{T} \sum_{t \in [T]} \mathbb{E} [\| \nabla \G(\x_t, \Blambda_t) \|^2]$.
The same applies for convex case as 
$\frac{1}{T} \sum_{t \in [T]} \mathbb{E} [\max_{\x_t^*} \min_{\Blambda_t^*} \left( G (\x_t, \Blambda_t^*) - G (\x_t^*, \Blambda_t^*) \right) ] \leq \frac{1}{T} \sum_{t \in [T]} \mathbb{E} [G (\x_t, \Blambda_t) - G (\x_*, \Blambda_t)]$.

% $G(\x_T, \Blambda_T) - G(\x_T^*, \Blambda_T)$

\textbf{Convergence Rates.}
As shown in Table~\ref{tab:rate}, we compare the convergence rates with baselines across strongly-convex, general convex and non-convex functions. 
We show our \psmg{} can achieve the state-of-the-art convergence rates in these cases, which match existing best results.

% \begin{table*}[h]
% \centering
% \renewcommand{\arraystretch}{1.5}
% \caption{Comparison of different algorithms for MOO problem in stochastic first order oracle.}
% \begin{tabular}{|c|c|c|c|}
% \hline
% \multirow{2}{*}{\bf Convexity} & \multirow{2}{*}{\bf Algorithm} & \multirow{1}{*}{\bf Assume Lipschitz} & \multirow{1}{*}{\bf Convergence} \\
% & & {\bf continuity of $\Blambda^*(\x)$} & {\bf Rate} \\
% \hline
% \multirow{4}{*}{Strongly Convex} & SMG~\citep{liu2021stochastic} & \ding{51} & $\mathcal{O}(\frac{1}{T})$ \\
% & MoDo~\citep{chen2024three} & \ding{55} & $\mathcal{O}(\frac{1}{T})$ \\
% & CR-MOGM~\citep{zhou2022convergence} & \ding{55} & $\mathcal{O}(\frac{1}{T})$ \\
% \rowcolor{gray!20} & \alg & \ding{55} & $\mathcal{O}(\frac{1}{T})$ \\
% \hline
% \multirow{3}{*}{General Convex} & SMG~\citep{liu2021stochastic} & \ding{51} & $\mathcal{O}(\frac{1}{\sqrt{T}})$ \\
% & CR-MOGM~\citep{zhou2022convergence} & \ding{55} & $\mathcal{O}(\frac{1}{\sqrt{T}})$ \\
%  \rowcolor{gray!20} & \alg & \ding{55} & $\mathcal{O}(\frac{1}{\sqrt{T}})$ \\
% \hline
% \multirow{5}{*}{Non-Convex} & CR-MOGM~\citep{zhou2022convergence} & \ding{55} & $\mathcal{O}(\frac{1}{\sqrt{T}})$ \\
%  & MoCO~\citep{fernando2022mitigating} & \ding{55} & $\mathcal{O}(\frac{1}{\sqrt{T}})$ \\
%  & MoDo~\citep{chen2024three} & \ding{55} & $\mathcal{O}(\frac{1}{\sqrt{T}})$ \\
%  & SDMGrad~\citep{xiao2024direction} & \ding{55} & $\mathcal{O}(\frac{1}{\sqrt{T}})$ \\
%  \rowcolor{gray!20} & \alg & \ding{55} & $\mathcal{O}(\frac{1}{\sqrt{T}})$ \\
% \hline
% \end{tabular}
% \label{tab:rate}
% \end{table*}

\subsection{Auxiliary Lemma}
\begin{lemma}[Lemma 2 and Lemma 7 in ~\cite{zhou2022convergence}]
    When we calculate the $\hat{\Blambda}_t^*$ by QP using stochastic gradient $\nabla f_s(\x_t, \xi_t)$, we have $\ee \left< \nabla \G(\x_t, \Blambda_t), - \hat{\nabla} \G(\x_t, \Blambda_t) \right> \leq 2SBH \sqrt{m \sigma^2 \mv [\Blambda_t]} - \ee [\| \nabla \G(\x_t, \Blambda_t) \|^2]$ and $\mv [\Blambda_t] := \ee [\| \Blambda_t - \mathbb{E} [\Blambda_t]  \|^2] \leq S^2 B^2 (1 - \alpha_t)^2.$
\end{lemma}

\begin{proof}
    \begin{align}
        \mv [\Blambda_t] &= \ee [\| \Blambda_t - \mathbb{E} [\Blambda_t] \|^2] \\
        &\leq \ee [\| \Blambda_t - \Blambda_{t - R} \|^2] \\
        &= \ee [\| (1 - \alpha_t) (\hat{\Blambda}_t^* - \Blambda_{t - R})  \|^2] \\
        &\leq (1 - \alpha_t)^2 \ee [\| (\hat{\Blambda}_t^* - \Blambda_{t - R}) \|^2] \\
        &\leq (1 - \alpha_t)^2 S^2 B^2 
    \end{align}
    where the first inequality is due to $\ee [\| \Blambda_t - \mathbb{E} [\Blambda_t] \|^2] \leq \ee [\| \Blambda_t - a \|^2]$ for any $a$ and the update of $\Blambda$ is recomputed every $R$ iterations; and the last inequality follows from the fact $0 \leq \hat{\Blambda}_{t,s}^*, \Blambda_{t,s} \leq B$.

    \small
    \begin{align}
        &\ee \left< \nabla \G(\x_t, \Blambda_t), - \hat{\nabla} \G(\x_t, \Blambda_t) \right> \\
        &= \ee \left< \nabla \G(\x_t, \Blambda_t), - \hat{\nabla} \G(\x_t, \Blambda_t) + \nabla \G(\x_t, \Blambda_t) - \nabla \G(\x_t, \Blambda_t) \right> \\
        &= \ee \left< \nabla \G(\x_t, \Blambda_t), - \hat{\nabla} \G(\x_t, \Blambda_t) + \nabla \G(\x_t, \Blambda_t) \right> \\
        &- \ee [\| \nabla \G(\x_t, \Blambda_t) \|^2] 
        % &\leq 2SBH \sqrt{m \sigma^2 \mv [\Blambda_t]} - \ee [\| \nabla \G(\x_t, \Blambda_t) \|^2]
    \end{align}

    \begin{align}
        &\ee \left< \nabla \G(\x_t, \Blambda_t), - \hat{\nabla} \G(\x_t, \Blambda_t) + \nabla \G(\x_t, \Blambda_t) \right> \\
        &= \ee \left< \nabla \G(\x_t, \Blambda_t), \sum_{s \in [S]} \Blambda_{t,s} (\nabla f_s(\x_t) -  \nabla f_s(\x_t, \xi_t) )\right> \\
        &= \underbrace{\ee \left< \nabla \G(\x_t, \Blambda_t), \sum_{s \in [S]} (\Blambda_{t,s} - \ee[\Blambda_{t,s}]) (\nabla f_s(\x_t) -  \nabla f_s(\x_t, \xi_t) )\right>}_{A_1} \\
        &\quad + \underbrace{\ee \left< \nabla \G(\x_t, \Blambda_t - \ee[\Blambda_t]), \sum_{s \in [S]} \ee[\Blambda_{t,s}] (\nabla f_s(\x_t) -  \nabla f_s(\x_t, \xi_t) )\right>}_{A_2} \\
        &\quad + \underbrace{\ee \left< \nabla \G(\x_t, \ee[\Blambda_t]), \sum_{s \in [S]} \ee[\Blambda_{t,s}] (\nabla f_s(\x_t) -  \nabla f_s(\x_t, \xi_t) )\right>}_{A_3}
    \end{align}

    \begin{align}
        &A_1 = \ee \left< \nabla \G(\x_t, \Blambda_t), \sum_{s \in [S]} (\Blambda_{t,s} - \ee[\Blambda_{t,s}]) (\nabla f_s(\x_t) -  \nabla f_s(\x_t, \xi_t) )\right> \\
        &\leq \ee \| \nabla \G(\x_t, \Blambda_t)\| \ee \left\|\sum_{s \in [S]} (\Blambda_{t,s} - \ee[\Blambda_{t,s}]) (\nabla f_s(\x_t) -  \nabla f_s(\x_t, \xi_t) ) \right\| \\
        &\leq SBH \ee \left[ \sum_{s \in [S]} | \Blambda_{t,s} - \ee[\Blambda_{t,s}] | \left\| (\nabla f_s(\x_t) -  \nabla f_s(\x_t, \xi_t) ) \right\| \right] \\
        &\leq \frac{SBH}{2 \epsilon} \sum_{s \in [S]} \ee \left[ | \Blambda_{t,s} - \ee[\Blambda_{t,s}] |^2 \right] \\
        &+ \frac{\epsilon SBH}{2} \sum_{s \in [S]} \ee \left[ \left\| (\nabla f_s(\x_t) -  \nabla f_s(\x_t, \xi_t) ) \right\|^2 \right] \\
        &\leq \frac{SBH}{2 \epsilon} \mv [\Blambda_t] + \frac{\epsilon S^2BH}{2} \sigma^2 \\
        &\leq SBH \sqrt{S \sigma^2 \mv [\Blambda_t]},
    \end{align}
    \normalsize
    where the last inequality follows from the fact $\epsilon = \frac{\sqrt{\mv [\Blambda_t]}}{\sqrt{S \sigma^2}}$.

    Similarly, using the same trick, we have the upper bound for $A_2$:
    \begin{align}
        A_2 \leq SBH \sqrt{m \sigma^2 \mv [\Blambda_t]}.
    \end{align}

    For $A_3$, it is easy to verify that $A_3 = 0$ as the randomness is only on the stochastic gradient which is unbiased estimations.

    Combining $A_1, A_2$ and $A_3$, we have 
        \begin{align}
        &\ee \left< \nabla \G(\x_t, \Blambda_t), - \hat{\nabla} \G(\x_t, \Blambda_t) \right> \\
        &\leq 2SBH \sqrt{m \sigma^2 \mv [\Blambda_t]} - \ee [\| \nabla \G(\x_t, \Blambda_t) \|^2]
    \end{align}
    
\end{proof}

\begin{lemma}[Lemma 1 in ~\cite{zhou2022convergence}]
    When we calculate the $\hat{\Blambda}_t^*$ by QP using stochastic gradient $\nabla f_s(\x_t, \xi_t)$, we have $\| \mathbb{E} [\hat{ \nabla} \G(\x_t, \Blambda_t)] - \mathbb{E} [\nabla \G(\x_t, \Blambda_t)] \|^2 \leq S \sigma^2 \mv [\Blambda_t] \leq \sigma^2 S^3 B^2 (1 - \alpha_t)^2.$
\end{lemma}

% \begin{proof}
%     TBD
% \end{proof}

% % -----------------
%  % PMG in non-convex
% % -------------------

% Non-convex
\subsection{Proof for Non-convex Functions}
\pmgdNC*

\begin{proof}
\textbf{Case I:}
When we calculate the $\hat{\Blambda}_t^*$ in one step ($t \% R == 0$), we have the following due to $L-$smoothness:
\small
\begin{align}
    &\ee [\G(\x_{t+1}, \Blambda_t)] - \G(\x_t, \Blambda_t) \\
    &\leq \ee \left< \nabla \G(\x_t, \Blambda_t), - \eta_t \hat{\nabla} \G(\x_t, \Blambda_t) \right> \\
    &+ \frac{1}{2} L \eta^2 \ee [\| \hat{ \nabla} \G(\x_t, \Blambda_t) \|^2] \\
    &\leq 2SBH \eta_t \sqrt{m \sigma^2 \mv [\Blambda_t]} - \eta_t \| \nabla \G(\x_t, \Blambda_t) \|^2 \\
    &+ \frac{1}{2} L \eta_t^2 \ee [\| \hat{ \nabla} \G(\x_t, \Blambda_t) \|^2]
\end{align}

\begin{align}
    &\ee [\| \hat{\nabla} \G(\x_t, \Blambda_t) \|^2] \\
    &= 2 \ee [\| \hat{\nabla} \G(\x_t, \Blambda_t) - \nabla \G(\x_t, \Blambda_t) \|^2] + 2 \| \nabla \G(\x_t, \Blambda_t) \|^2 \\
    &\leq 2 \ee [\| \sum_{s \in [S]} \lambda_t^s (\nabla f_s(\x_t, \xi_t) - \nabla f_s(\x_t)) \|^2] + 2 \| \nabla \G(\x_t, \Blambda_t) \|^2 \\
    &\leq 2 \ee \left[ \left( \sum_{s \in [S]} \lambda_t^s \| (\nabla f_s(\x_t, \xi_t) - \nabla f_s(\x_t) \| \right)^2 \right] + 2 \| \nabla \G(\x_t, \Blambda_t) \|^2 \\
    &\leq 2 S^2B^2 \sigma^2 + 2 \| \nabla \G(\x_t, \Blambda_t) \|^2
\end{align}

% \small
\begin{align}
    &\ee [\G(\x_{t+1}, \Blambda_t)] - \G(\x_t, \Blambda_t) \\
    &\leq \ee \left< \nabla \G(\x_t, \Blambda_t), - \eta_t \hat{\nabla} \G(\x_t, \Blambda_t) \right> + \frac{1}{2} L \eta^2 \ee [\| \hat{ \nabla} \G(\x_t, \Blambda_t) \|^2] \\
    &= 2SBH \eta_t \sqrt{S \sigma^2 \mv [\Blambda_t]} - \eta_t (1 - L \eta_t) \| \nabla \G(\x_t, \Blambda_t) \|^2 + L \eta_t^2 S^2B^2 \sigma^2 \\
    &\leq 2S^{3/2}B^2H (1 - \alpha_t) \eta_t \sigma  - \eta_t (1 - L \eta_t) \| \nabla \G(\x_t, \Blambda_t) \|^2 + L \eta_t^2 S^2B^2 \sigma^2
\end{align}
\normalsize
If $\frac{1}{2L} \leq \eta_t < \frac{1}{L}$, then
% \small
\begin{align}
    &\| \nabla \G(\x_t, \Blambda_t) \|^2 \\
    &\leq \frac{2 [ \G(\x_t, \Blambda_t) - \ee[\G(\x_{t+1}, \Blambda_t)] ]}{\eta_t}\\
    &+\frac{4S^{3/2}B^2H (1 - \alpha_t) \eta_t \sigma + 2L \eta_t^2 S^2B^2 \sigma^2}{\eta_t}
    \label{ineq: nabla G}
\end{align}
% \normalsize
\textbf{Case II:}
When we reuse the coefficients from previous steps $\Blambda_t = \Blambda_{t - \tau}$ ($(t - \tau) \%R == 0, t \% R \neq 0, \tau < R$), we have the following:
\small
\begin{align}
    &\ee [\G(\x_{t+1}, \Blambda_t)] - \G(\x_t, \Blambda_t) \\
    &\leq \ee \left< \nabla \G(\x_t, \Blambda_t), - \eta_t \hat{\nabla} \G(\x_t, \Blambda_t) \right> + \frac{1}{2} L \eta^2 \ee [\| \hat{ \nabla} \G(\x_t, \Blambda_t) \|^2] \\
    &\leq - \eta_t \| \nabla \G(\x_t, \Blambda_t) \|^2 + \frac{1}{2} L \eta_t^2 \ee [\| \hat{ \nabla} \G(\x_t, \Blambda_t) \|^2]
\end{align}
\normalsize
where the last inequality follows from the fact that $\ee [\hat{\nabla} \G(\x_t, \Blambda_t)] = \nabla \G(\x_t, \Blambda_t)$ as a fixed $\Blambda_t$ is used and independent of current step sampling.

Similarly, if $\frac{1}{2L} \leq \eta_t < \frac{1}{L}$, then
\small
\begin{align}
    \| \nabla \G(\x_t, \Blambda_t) \|^2 &\leq \frac{2 [\G(\x_t, \Blambda_t) - \ee[\G(\x_{t+1}, \Blambda_t)] ] + 2L \eta_t^2 S^2B^2 \sigma^2}{\eta_t}
\end{align}
\normalsize

For non-increasing learning rate $\eta_t$,
\small
\begin{align}
    &\sum_{t \in [T]} \frac{\mathbb{E} [\G(\x_t, \Blambda_t) - \G(\x_{t+1}, \Blambda_t)]}{\eta_t} \\
    &= \sum_{t \%R == 0} \frac{\mathbb{E} [\G(\x_t, \Blambda_t) - \G(\x_{t+R}, \Blambda_t)]}{\eta_t} \\
    &= \frac{\mathbb{E} [\G(\x_0, \Blambda_0) - \G(\x_T, \Blambda_T)]}{\eta_T} + \sum_{t \%R == 0, t \neq 0} \frac{\mathbb{E} [ \G(\x_t, \Blambda_t - \Blambda_{t-R}) ]}{\eta_t} \\
    &= \frac{\mathbb{E} [\G(\x_0, \Blambda_0) - \G(\x_T, \Blambda_T)]}{\eta_T} \\
    &+ \sum_{t \%R == 0, t \neq 0} \frac{ \mathbb{E} [ \G(\x_t, (1 - \alpha_t ) (\hat{\Blambda}_{t} - \Blambda_{t-R}) ) ]}{\eta_t} \\
    &\leq \frac{2SBF}{\eta_T} + F \sum_{t \%R == 0, t \neq 0} \mathbb{E} \left[ \sum_{s \in [S]} (1 - \alpha_t ) | \hat{\lambda}_{t}^s - \lambda_{t-R}^s | \right]
    % &\leq 2SBF + SB F \sum_{t \%R == 0, t \neq 0} (1 - \alpha_t ),
\end{align}
\normalsize
where the last inequality is due to $ | f_s(\x) | \leq F$.

Telescoping and rearranging, 
\begin{align}
    & \frac{1}{T} \sum_{t \in [T]} \mathbb{E} \| \nabla \G(\x_t, \Blambda_t) \|^2 \\
    &\leq \frac{1}{T} \sum_{t \in [T]} \frac{2 [\G(\x_t, \Blambda_t) - \ee[\G(\x_{t+1}, \Blambda_t)] ] + 2L \eta_t^2 S^2B^2 \sigma^2}{\eta_t} \\
    &+ \frac{4S^{3/2}B^2H \sigma}{T} \sum_{t \% R == 0} (1 - \alpha_t) \\
    &\leq \frac{4SBF}{T \eta_T} + \frac{2F}{T} \sum_{t \%R == 0, t \neq 0} \frac{ \mathbb{E} \left[ \sum_{s \in [S]} (1 - \alpha_t ) | \hat{\lambda}_{t}^s - \lambda_{t-R}^s | \right]}{\eta_t} \\
    &\quad + 2LS^2B^2 \sigma^2 \frac{1}{T} \sum_{t \in [T]} \eta_t + \frac{4S^{3/2}B^2H \sigma}{T} \sum_{t \% R == 0} (1 - \alpha_t)
\end{align}

By setting $1 - \alpha_t = \min \{\frac{\eta_t}{\eta_1}, \frac{\eta_t}{\eta_1 \sqrt{t} \max_{s \in [S]} | \hat{\lambda}_{t}^s - \lambda_{t-R}^s |} \},$ and $\eta_t = \mathcal{O}(\frac{1}{\sqrt{t}})$, the convergence rate is
$\frac{1}{T} \sum_{t \in [T]} \mathbb{E} \| \nabla \G(\x_t, \Blambda_t) \|^2 = \mathcal{O}(\frac{1}{\sqrt{T}}).$

\end{proof}

% ------------------
    %  convex
% ------------------

% convex
\subsection{Proof for General Convex Functions}
\pmgdGC*

\begin{proof}
\begin{align}
    &\mathbb{E} [\| \x_{t+1} - \x_* \|^2] = \mathbb{E} [\| \x_t - \eta_t \hat{ \nabla} \G(\x_t, \Blambda_t) - \x_* \|^2] \\
    &= \| \x_t - \x_* \|^2 + \eta_t^2 \mathbb{E} [\| \hat{ \nabla} \G(\x_t, \Blambda_t) \|^2] \\
    &- 2 \eta_t \mathbb{E} \left< \hat{ \nabla} \G(\x_t, \Blambda_t), (\x_t - \x_*) \right> \\
    &= \| \x_t - \x_* \|^2 + \eta_t^2 \mathbb{E} [\| \hat{ \nabla} \G(\x_t, \Blambda_t) \|^2] \\
    &- 2 \eta_t \left< \mathbb{E} [\hat{ \nabla} \G(\x_t, \Blambda_t)] - \mathbb{E} [\nabla \G(\x_t, \Blambda_t)], (\x_t - \x_*) \right> \\
    & \quad
    - 2 \eta_t \left< \mathbb{E} [\nabla \G(\x_t, \Blambda_t)], (\x_t - \x_*) \right> 
\end{align}

Due to convexity, we have $\left< \nabla \G(\x_t, \Blambda_t), (\x_t - \x_*) \right> \geq \G (\x_t, \Blambda_t) - \G (\x_*, \Blambda_t)$.

\textbf{Case I:}
When we calculate the $\hat{\Blambda}_t^*$ in one step ($t \% R == 0$), $\mathbb{E} [\hat{ \nabla} \G(\x_t, \Blambda_t)] - \mathbb{E} [\nabla \G(\x_t, \Blambda_t)] \neq 0$ due to dependence of $\Blambda_t$ on the samples.
\small
\begin{align}
    & 2 \eta_t [\G (\x_t, \Blambda_t) - \G (\x_*, \Blambda_t)] \leq - \mathbb{E} [\| \x_{t+1} - \x_* \|^2] + \| \x_t - \x_* \|^2 \\
    &+ \eta_t^2 \mathbb{E} [\| \hat{ \nabla} \G(\x_t, \Blambda_t) \|^2] \\
    &- 2 \eta_t \left< \mathbb{E} [\hat{ \nabla} \G(\x_t, \Blambda_t)] - \mathbb{E} [\nabla \G(\x_t, \Blambda_t)], (\x_t - \x_*) \right> \\ 
    &\leq - \mathbb{E} [\| \x_{t+1} - \x_* \|^2] + \| \x_t - \x_* \|^2 + \eta_t^2 \mathbb{E} [\| \hat{ \nabla} \G(\x_t, \Blambda_t) \|^2] \\
    &+ 2 \eta_t \| \mathbb{E} [\hat{ \nabla} \G(\x_t, \Blambda_t)] - \mathbb{E} [\nabla \G(\x_t, \Blambda_t)] \| \| (\x_t - \x_*) \| \\
    &\leq - \mathbb{E} [\| \x_{t+1} - \x_* \|^2] + \| \x_t - \x_* \|^2 + 2 S^2B^2 \eta_t^2 \sigma^2 \\
    &+ 2 \eta_t^2 \| \nabla \G(\x_t, \Blambda_t) \|^2 + 2 \eta_t D \sigma S^{3/2} B (1 - \alpha_t) \\
    &\leq - \mathbb{E} [\| \x_{t+1} - \x_* \|^2] + \| \x_t - \x_* \|^2 + 2 S^2B^2 \eta_t^2 \sigma^2 + 2 \eta_t^2 S^2 B^2 H^2 \\
    &+ 2 \eta_t D \sigma S^{3/2} B (1 - \alpha_t)
\end{align}
% \normalsize
% $\| \x - \x_* \| \leq D$

\begin{align}
    & \G (\x_t, \Blambda_t) - \G (\x_*, \Blambda_t) \leq \frac{- \mathbb{E} [\| \x_{t+1} - \x_* \|^2] + \| \x_t - \x_* \|^2}{2 \eta_t} \\
    &+ \eta_t S^2B^2 \sigma^2 + \eta_t S^2 B^2 H^2 + D \sigma S^{3/2} B (1 - \alpha_t)
\end{align}
\normalsize

\textbf{Case II:}
When we reuse the coefficients from previous steps $\Blambda_t = \Blambda_{t - \tau}$ ($(t - \tau) \%R == 0, t \% R \neq 0, \tau < R$), $\mathbb{E} [\hat{ \nabla} \G(\x_t, \Blambda_t)] - \mathbb{E} [\nabla \G(\x_t, \Blambda_t)] = 0$.

\small
\begin{align}
    & 2 \eta_t [\G (\x_t, \Blambda_t) - \G (\x_*, \Blambda_t)] \\
    &\leq - \mathbb{E} [\| \x_{t+1} - \x_* \|^2] + \| \x_t - \x_* \|^2 + \eta_t^2 \mathbb{E} [\| \hat{ \nabla} \G(\x_t, \Blambda_t) \|^2] \\ 
    &\leq - \mathbb{E} [\| \x_{t+1} - \x_* \|^2] + \| \x_t - \x_* \|^2 + 2 S^2B^2 \eta_t^2 \sigma^2 + 2 \eta_t^2 S^2 B^2 H^2
\end{align}
\normalsize

That is,

\small
\begin{align}
    &\G (\x_t, \Blambda_t) - \G (\x_*, \Blambda_t) \\
    &\leq \frac{- \mathbb{E} [\| \x_{t+1} - \x_* \|^2] + \| \x_t - \x_* \|^2}{2 \eta_t} + \eta_t S^2B^2 \sigma^2 + \eta_t S^2 B^2 H^2
\end{align}
\normalsize

Assume $\eta_t = \frac{1}{\sqrt{t}}$ and $(1 - \alpha_t) = \eta_t$, combining two cases and telescoping:

\begin{align}
    &\frac{1}{T} \sum_{t \in [T]} \mathbb{E} [\G (\x_t, \Blambda_t) - \G (\x_*, \Blambda_t)] \\
    &\leq \frac{\| \x_1 - \x_* \|^2}{T \eta_1} + (2 S^2B^2 \sigma^2 + 2 S^2 B^2 H^2) \frac{1}{T} \sum_{t \in [T]} \eta_t \\
    &+ 2 D \sigma S^{3/2} B \frac{1}{T} \sum_{t \% R == 0} (1 - \alpha_t) \\
    &= \mathcal{O}(1/\sqrt{T})
\end{align}
\end{proof}

% ------------------
    %  strongly convex
% ------------------

% strongly convex
\subsection{Proof for Strongly-Convex Functions}
\pmgdSC*

\begin{proof}

\textbf{Case I:}
When we calculate the $\hat{\Blambda}_t^*$ in one step ($t \% R == 0$), $\mathbb{E} [\hat{ \nabla} \G(\x_t, \Blambda_t)] - \mathbb{E} [\nabla \G(\x_t, \Blambda_t)] \neq 0$ due to dependence of $\Blambda_t$ on the samples. Due to the strongly-convex property of $f_s, s \in [S]$,

\begin{align}
    &2 \mu \ee (\G(\x_t, \Blambda_t) - \G(\x_t^*, \Blambda_t)) \leq \ee \| \nabla \G(\x_t, \Blambda_t) \|^2 \\
    &\leq \frac{2 [ \G(\x_t, \Blambda_t) - \ee[\G(\x_{t+1}, \Blambda_t)] ] + 4S^{3/2}B^2H (1 - \alpha_t) \eta_t \sigma}{\eta_t} \\
    & + \frac{2L \eta_t^2 S^2B^2 \sigma^2}{\eta_t}
\end{align}
The second inequality is due to ~\eqref{ineq: nabla G}, where $\x_t^* = \argmin_{\x} \G(\x, \Blambda_t)$.

Rearranging it, we have 
\begin{align}
    &\ee[\G(\x_{t+1}, \Blambda_t) - \G(\x_t^*, \Blambda_t)]  \\
    &\leq  (1 - 2 \mu \eta_t) \ee (\G(\x_t, \Blambda_t) - \G(\x_t^*, \Blambda_t)) \\ & + 4S^{3/2}B^2H (1 - \alpha_t) \eta_t \sigma + 2L \eta_t^2 S^2B^2 \sigma^2.
\end{align}

We know that the $\Blambda$ will be reused for $t+1 \% \neq 0$. That is, $\Blambda_{t+1} = \Blambda_t$ and $\x_{t+1}^* = \x_t^*$.

\begin{align}
    &\ee[\G(\x_{t+1}, \Blambda_{t+1}) - \G(\x_{t+1}^*, \Blambda_{t+1})]  \\
    &\leq  (1 - 2 \mu \eta_t) \ee (\G(\x_t, \Blambda_t) - \G(\x_t^*, \Blambda_t)) \\
    & + 4S^{3/2}B^2H (1 - \alpha_t) \eta_t \sigma + 2L \eta_t^2 S^2B^2 \sigma^2.
\end{align}

\textbf{Case II:}
When we reuse the coefficients from previous steps $\Blambda_t = \Blambda_{t - \tau}$ ($(t - \tau) \%R == 0, t \% R \neq 0, \tau < R$), $\mathbb{E} [\hat{ \nabla} \G(\x_t, \Blambda_t)] = \nabla \G(\x_t, \Blambda_t)$.

\begin{align}
    &2 \mu \ee (\G(\x_t, \Blambda_t) - \G(\x_t^*, \Blambda_t)) \leq \ee \| \nabla \G(\x_t, \Blambda_t) \|^2 \\
    &\leq \frac{2 [\G(\x_t, \Blambda_t) - \ee[\G(\x_{t+1}, \Blambda_t)] ] + 2L \eta_t^2 S^2B^2 \sigma^2}{\eta_t}
\end{align}
The last inequality follows from the fact that
\begin{align}
    &\| \nabla \G(\x_t, \Blambda_t) \|^2 \\
    &\leq \frac{2 [\G(\x_t, \Blambda_t) - \ee[\G(\x_{t+1}, \Blambda_t)] ] + 2L \eta_t^2 S^2B^2 \sigma^2}{\eta_t}
\end{align}

Rearranging it, we have 
\begin{align}
    &\ee[\G(\x_{t+1}, \Blambda_t) - \G(\x_t^*, \Blambda_t)]  \\
    &\leq  (1 - 2 \mu \eta_t) \ee (\G(\x_t, \Blambda_t) - \G(\x_t^*, \Blambda_t)) + 2L \eta_t^2 S^2B^2 \sigma^2
\end{align}

For $t+1 \% R \neq 0$, we know that $\Blambda_{t+1} = \Blambda_t$ and $\x_{t+1}^* = \x_t^*$ by reusing $\Blambda$, then 
\begin{align}
    &\ee[\G(\x_{t+1}, \Blambda_{t+1}) - \G(\x_t^*, \Blambda_{t+1})]  \\
    &\leq  (1 - 2 \mu \eta_t) \ee (\G(\x_t, \Blambda_t) - \G(\x_t^*, \Blambda_t)) + 2L \eta_t^2 S^2B^2 \sigma^2
\end{align}

For $t+1 \% R = 0$, we will recalculate $\Blambda_{t+1}$.

\small
\begin{align}
    &\G(\x_{t+1}, \Blambda_{t+1}) - \G(\x_{t+1}^*, \Blambda_{t+1}) \\
    &= \G(\x_{t+1}, \Blambda_t) - \G(\x_{t+1}^*, \Blambda_t) - \G(\x_{t+1}, \Blambda_t - \Blambda_{t+1}) \\
    &+ \G(\x_{t+1}^*, \Blambda_t - \Blambda_{t+1}) \\
    &\leq \G(\x_{t+1}, \Blambda_t) - \G(\x_{t+1}^*, \Blambda_t) + 2 F (1 - \alpha_t) \sum_{s \in [S]}  | \hat{\lambda}_{t+1}^s - \lambda_t^s | \\
    &\leq (1 - 2 \mu \eta_t) \ee (\G(\x_t, \Blambda_t) - \G(\x_t^*, \Blambda_t)) \\
    &+ 2L \eta_t^2 S^2B^2 \sigma^2 + 2 F (1 - \alpha_t) \sum_{s \in [S]}  | \hat{\lambda}_{t+1}^s - \lambda_t^s | \\
    &\leq (1 - 2 \mu \eta_t) \ee (\G(\x_t, \Blambda_t) - \G(\x_t^*, \Blambda_t)) + 2L \eta_t^2 S^2B^2 \sigma^2 \\
    &+ 4 F (1 - \alpha_t) SB
\end{align}

\begin{align}
    &\G(\x_{t+1}, \Blambda_{t+1}) - \G(\x_{t+1}^*, \Blambda_{t+1}) \\
    &\leq (1 - 2 \mu \eta_t) \ee (\G(\x_t, \Blambda_t) - \G(\x_t^*, \Blambda_t)) + 2L \eta_t^2 S^2B^2 \sigma^2 \\
    &\quad  + 4 F (1 - \alpha_t) SB \mathbf{1} \{t+1 \% R = 0\} \\
    &\quad  + 4S^{3/2}B^2H (1 - \alpha_t) \eta_t \sigma \mathbf{1} \{t \% R = 0\} \\
    &\quad \leq (1 - 2 \mu \eta_t) \ee (\G(\x_t, \Blambda_t) - \G(\x_t^*, \Blambda_t)) \\
    &\quad  + \eta_t^2 (2L S^2B^2 \sigma^2  + 4 F SB \mathbf{1} \{t+1 \% R = 0\} \\
    & \quad  + 4S^{3/2}B^2H \sigma \mathbf{1} \{t \% R = 0\})
\end{align}

\begin{align}
    &\Phi = (2L S^2B^2 \sigma^2  + 4 F SB \mathbf{1} \{t+1 \% R = 0\} \\
    &+ 4S^{3/2}B^2H \sigma \mathbf{1} \{t \% R = 0\}
\end{align}
\normalsize
where we set $(1 - \alpha_t) = \eta_t$.

Using $\eta_t = \frac{c}{T}$ with $c > \frac{1}{2 \mu}$, 
\begin{align}
    &\G(\x_{T}, \Blambda_{T}) - \G(\x_{T}^*, \Blambda_{T}) \\
    &\leq \frac{\max \{2 c^2 \Phi (2\mu c - 1)^{-1}, \G(\x_0, \Blambda_0) - \G(\x_0^*, \Blambda_0) \}}{T}
\end{align}

\end{proof}
\section{Supplement to Experiments}
\label{sec:sup_ex}

% \subsection{Experimental Setting}

% All experiments except the Multi-MNIST in this paper are conducted using PyTorch on a single NVIDIA A100 GPU on the Slurm system of~\cite{RIT2019}. The Multi-MNIST experiment is conducted on a single NVIDIA GeForce RTX 2070 Super GPU under Ubuntu 18.04 LTS.
\subsection{Supplement convergence experiments for the QM9 and NYU-v2}
\label{sec:sup_pre}
Here are convergence experiments results for the QM9 and NYU-v2. Figure~\ref{fig:QM9loss} illustrates the test loss trajectories over 300 training epochs for the QM-9 dataset, where our \psmg{} is tasked with predicting various molecular properties. The figure compares the performance of \psmg{} with other MOO algorithms. The test loss in Figure~\ref{fig:QM9loss}, computed as the average MAE across all molecular property prediction tasks, provides a clear metric of the convergence performance. The results demonstrate that \psmg{} achieves a smoother and faster reduction in test loss, particularly during the early stages of training, showcasing its efficiency and robustness in handling the MOO challenges presented by the QM-9 dataset.

\begin{figure}[htbp]
    \centering
    \includegraphics[width=\columnwidth]{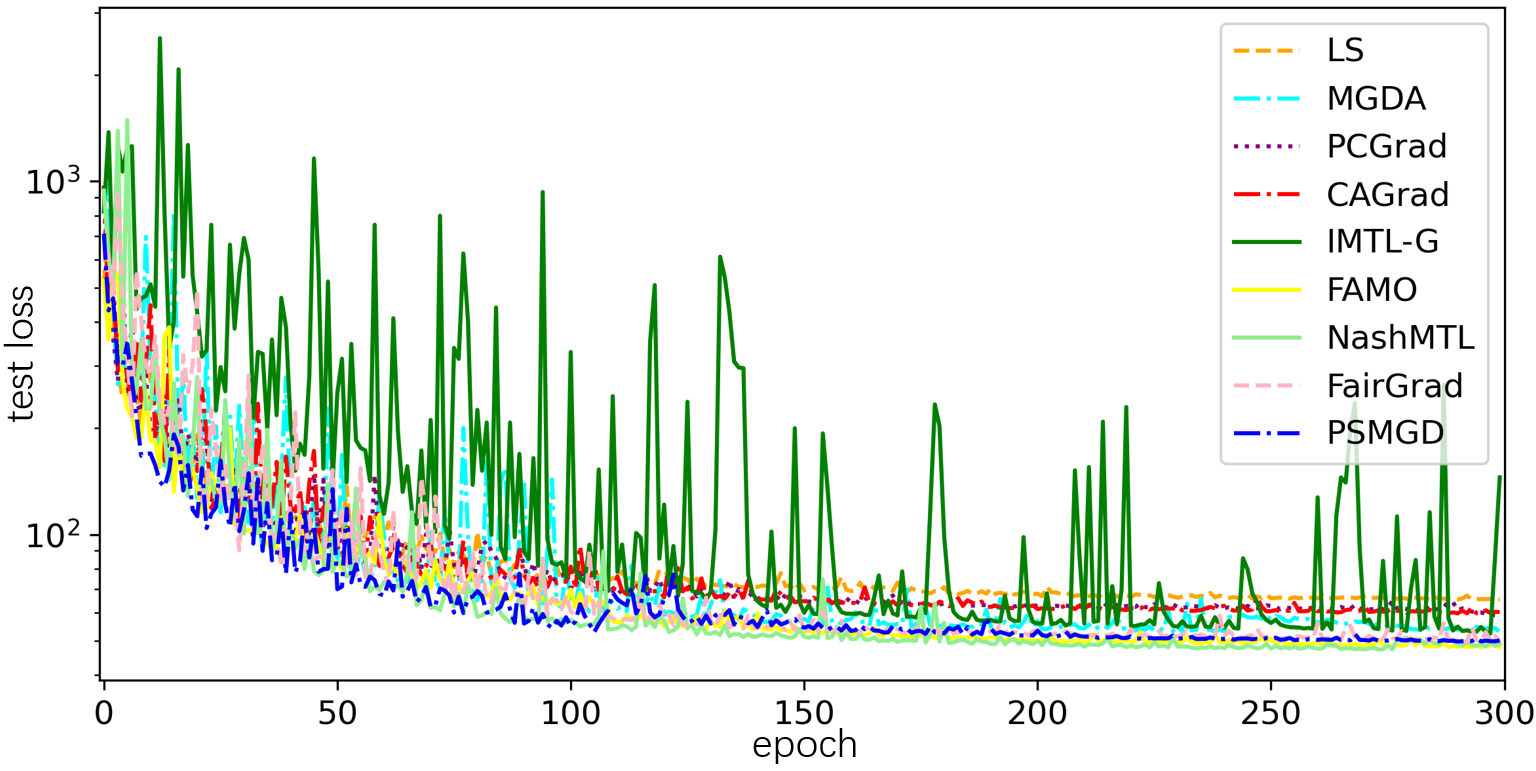}

    \caption{Test loss in training (300 epochs) on QM-9.
    }
    \label{fig:QM9loss}
\end{figure}

Figure~\ref{fig:nyuv2_loss} presents the test loss curves across 200 training epochs for three dense prediction tasks, including image segmentation, depth prediction, and surface normal estimation on the NYU-v2 dataset. It provides a comparative analysis of \psmg{} against other gradient-based multi-objective optimization methods, highlighting its superior performance in reducing test loss across all tasks. The consistent decrease in loss for each task, particularly in the early training epochs, shows \psmg{}'s ability to effectively manage and optimize multiple objectives simultaneously, making it a excellent algorithm for complex MOO scenarios such as those encountered in dense prediction datasets like the NYU-v2. Through the realistic convergence experiments here, we can clearly find \psmg{} can achieve an outstanding convergence performance with advanced performance metrics in all tasks.

\begin{figure*}[htbp]
    \centering
    \includegraphics[width=\textwidth]{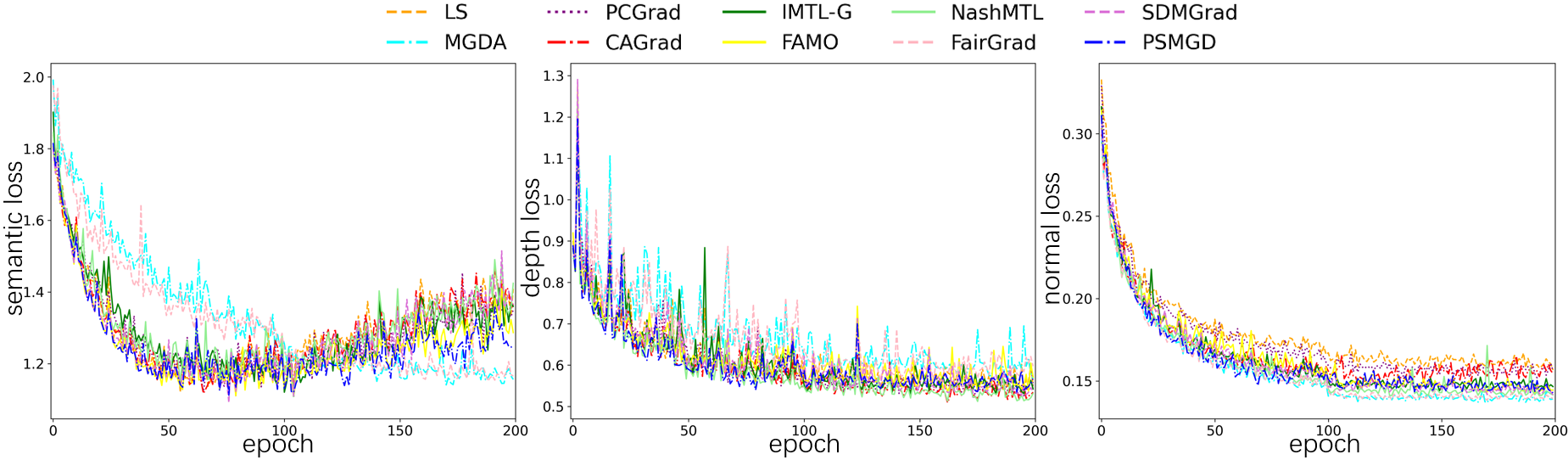}
    \caption{Test loss of segmentation, depth, and surface normal in training (200 epochs) on NYU-v2.
    }
    \label{fig:nyuv2_loss}
\end{figure*}

\subsection{Pre-Study: Multi-MNIST}
\label{sec:Multi}

We present experiments on the Multi-MNIST dataset~\cite{sener2018multi}, a classical 2-task supervised learning benchmark. It is constructed by uniformly sampling MNIST~\citep{lecun1998gradient} images and placing one in the top-left corner and the other in the bottom-right corner. To ensure consistency with the experimental preparation of~\cite{kurin2022defense}, we utilize a modified encoder-decoder version of the LeNet architecture~\citep{lecun1998gradient}. The cross-entropy loss is utilized for both tasks. All methods are trained in 100 epochs using Adam~\citep{kingma2014adam} in the stochastic gradient setting, with an initial learning rate of $1e-2$. They are tuned in $\{1e-3, 1e-2, 1e-1\}$ and yield the best validation results for all optimizers under consideration, exponentially decayed by 0.95 after each epoch with a batch size of 256. Note that, in the Multi-MNIST, every method is repeated 10 times using varying random seeds with the mean results reported. 
% This allows for a detailed analysis of training performance, overall accuracy, task-specific accuracy distribution, weight adaptation, etc.

Figure~\ref{fig:acc} displays the average test accuracy and the training time per epoch and iteration. The test model was chosen in each run based on the highest average validation accuracy across the configured epochs. In Figure~\ref{fig:acc}(b), the overlapping confidence intervals of average test accuracy indicate that none of the considered algorithms outperforms the others. Their test accuracies in 10 runs of the training process are between 0.935 and 0.95. 
% However, linear scalarization exhibits greater experimental variability. 
% For \psmg{}, its average test accuracy is not the highest, but it's stable within a certain range.

In Figure~\ref{fig:acc}(b) and (c), we conclude the results of the time consumption per epoch and per iteration. They also indicate that \psmg{} has among the lowest training times, both per epoch and per iteration. Heading further into a training epoch, we discover that linear sacralization requires the least amount of time. \psmg{} takes $1.44s$, which is slightly longer than linear sacralization but shorter than the other methods. Additionally, we can see that \psmg{} only takes $3.58ms$, the shortest time for a training iteration, and also second only to linear sacralization.

However, in Figure~\ref{fig:mnist}(b), linear scalarization requires the most epochs to reach convergence (validation accuracy $\ge$ 0.9), resulting in a longer time despite having only one backpropagation per iteration. \psmg{} boasts minimal backpropagation complexity and convergence time due to its streamlined but efficient structure. In particular, in Figure~\ref{fig:mnist}(c), \psmg{} takes $7.19s$, the shortest average convergence time consumption, to achieve convergence in a single run.

In Figure~\ref{fig:mnist}(a), we also observe that \psmg{} has a wider accuracy distribution in both Task 1 and Task 2. We use linear scalarization with grid search to simulate the approximate curve of the Pareto front for the Multi-MNIST dataset. Here, \psmg{}'s task accuracies exhibit a broader distribution, with a pattern that encompasses single-task accuracies adjacent to task 1 and task 2, as well as their overlapping center regions.

Moreover, we also explore the weight adaptation here to validate the motivation of our algorithm \psmg{}. During the MGDA training process, weights are updated regularly. Figure~\ref{fig:weights}(a) shows that the weights for task 1 and task 2 have a consistent average value and low variability. This indicates that the weight vectors are updated minimally throughout the training process. Delving into an epoch, in Figure~\ref{fig:weights}(b), it is evident that most weight updates fall within the range of [0.4, 0.6] per iteration. This suggests that MGDA does not make drastic changes to the weight vectors. This is precisely why we can update the weight vectors every $R$ iteration.

\begin{figure*}[htbp]
    \centering
    \includegraphics[width=\textwidth]{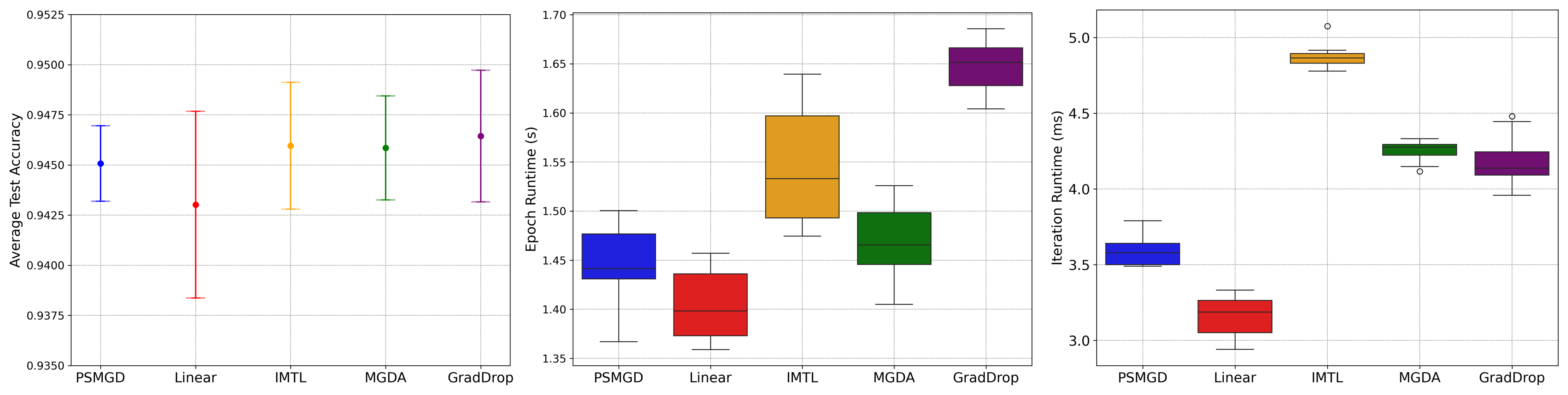}
    \caption{\psmg{} achieves a favorable balance between test accuracy and training time of epoch and iteration. (a) (\textit{Left}) Average test accuracy: mean and 95\% CI (10 runs). (b) (\textit{Middle}) Box plots for the average training time of an epoch (10 runs). (c) (\textit{Right}) Box plots for the average training time of an iteration (10 epochs).
    }
    \label{fig:acc}
\end{figure*}

\begin{figure*}[htbp]
    \centering
    \includegraphics[width=\textwidth]{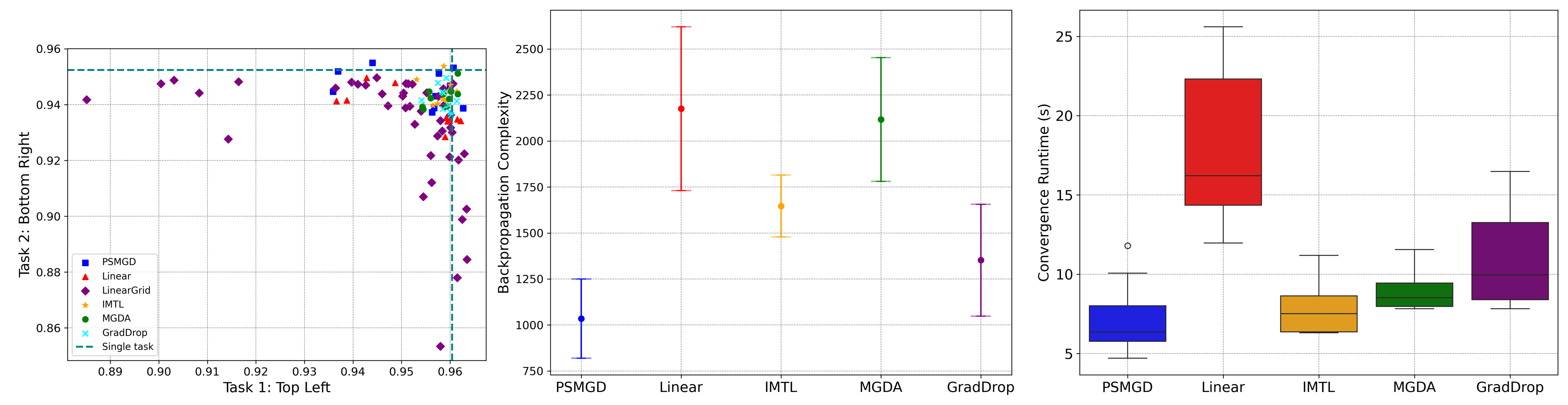}
    \caption{\psmg{} achieves a favorable balance between each task's accuracy and convergence complexity time of epoch in MultiMNIST. (a) (\textit{Left}) The results for MultiMNIST with Task1\&2 accuracy. (b) (\textit{Middle}) Average backpropagation complexity for convergence: mean and 95\% CI (10 runs). (c) (\textit{Right}) Box plots for convergence time (10 runs).
    }
    \label{fig:mnist}
\end{figure*}

\begin{figure}[htbp]
    \centering
    \begin{minipage}[b]{0.49\textwidth}
        \centering
        \includegraphics[width=1\textwidth]{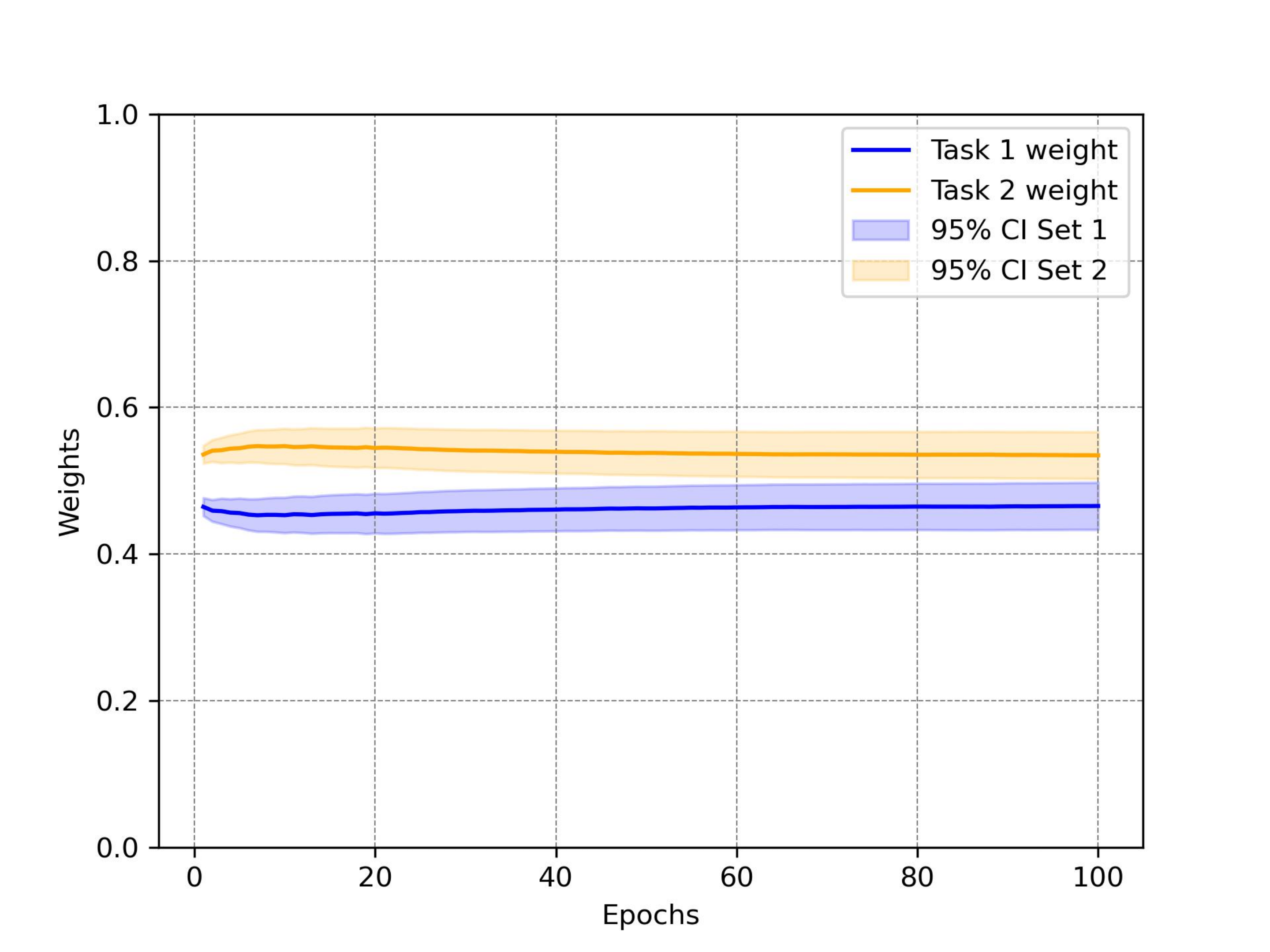}
    \end{minipage}
    \hfill
    \begin{minipage}[b]{0.49\textwidth}
    \centering
    \includegraphics[width=1\textwidth]{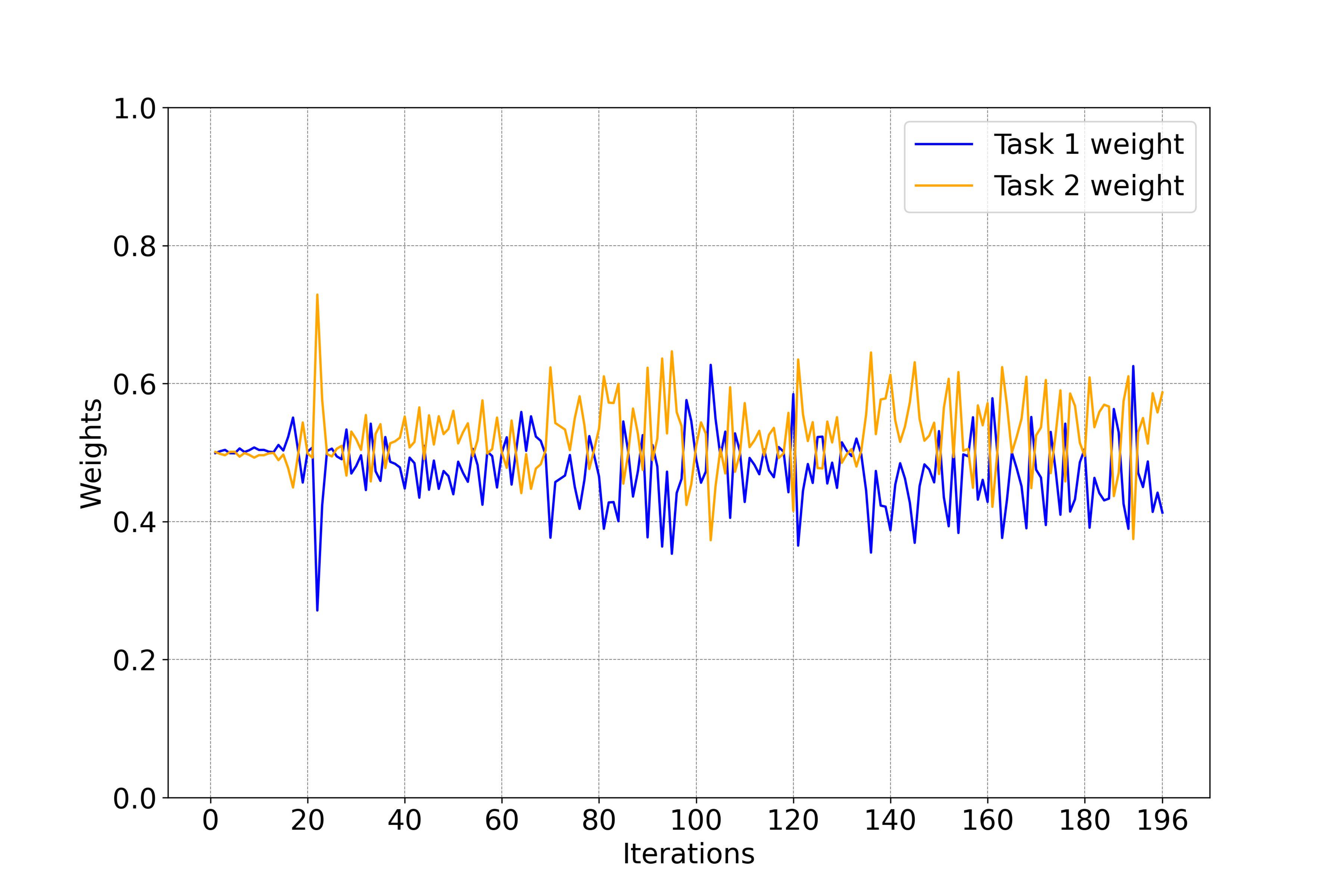}
    \end{minipage}
    \hfill
    % \begin{minipage}[b]{0.325\textwidth}
    % \centering
    % \includegraphics[width=1\textwidth]{images/mnist_conver_time}
    % \end{minipage}
	\caption{The dynamic weight vectors for MGDA and linear scalarization during the training process for the MultiMNIST. (a) (\textit{Upper}) The results for dynamic weight vectors for MGDA in a run (100 epochs). (b) (\textit{Lower}) The results for dynamic weight vectors for MGDA in an epoch (196 iterations).}
	\label{fig:weights}
\end{figure}

\subsection{Additional Supervised Learning Experiments: CityScapes and CelebA}

\textbf{Experiment Environments.}
1. Hyper-parameters.
For our \psmg{}, we choose the best hyperparameter $R \in \{4, 8, 16\}$ and momentum $\Blambda=0.9$ based on the validation loss in experiments. Specifically, we choose $R=4$ for the Multi-MNIST dataset in Section~\ref{sec:Multi} and $R=8$ for the rest of the datasets in Section~\ref{sec: exp}. 
The hyper-parameters in the baselines follow from the settings in their corresponding papers, which align with existing works.
% Following the standard training pipeline, we utilize validation metrics to identify the optimal model and then report test metrics for all MOO algorithms.
2. Software and Hardware. 
Following experiment settings from~\citep{liu2024famo,ban2024fair,xiao2024direction}, all experiments except the Multi-MNIST are conducted using PyTorch on a single NVIDIA A100 GPU.
% on the Slurm system of~\citep{RIT2019}. 
The Multi-MNIST experiment is conducted on a single NVIDIA GeForce RTX 2070 Super GPU under Ubuntu 18.04 LTS.

CityScapes dataset~\cite{cordts2016cityscapes} is a comprehensive collection of high-quality street-scene images, specifically designed for the evaluation of computer vision algorithms in urban environments. It comprises 5K finely annotated images captured across different cities. The dataset supports 2 tasks: semantic segmentation, which involves classifying each pixel in an image into one of 30 classes grouped into 7 super-categories, and depth estimation, which involves predicting the distance of objects from the camera. These 2 tasks provide a robust benchmark for evaluating the effectiveness of MOO methods in handling complex urban scenes, characterized by a diverse range of objects, varying lighting conditions, and dynamic elements like pedestrians and vehicles. We follow experiment settings from~\cite{liu2024famo,ban2024fair,xiao2024direction} and adopt the backbone of MTAN~\cite{liu2019end}, adding objective-specific attention modules to SegNet~\cite{badrinarayanan2017segnet}. Our method is trained for 200 epochs with batch size 8 on the CityScapes. The initial learning rate is set to $1e-4$ for the first 100 epochs and then halved for the remaining epochs.

CelebA dataset~\cite{liu2015deep} is a large-scale face attributes dataset, consisting of over 200K images of celebrities, annotated with 40 different attributes. These attributes encompass a wide range of facial features and expressions, such as smiling, wavy hair, mustache, and many more. Each image is labeled with binary indicators for the presence or absence of these attributes, making CelebA a valuable resource for MOO research. In this context, it can be viewed as an image-level 40-task MTL classification problem, where each task involves predicting the presence of a specific attribute. This setup provides a comprehensive benchmark for evaluating the performance of MTL methods in managing a substantial number of tasks simultaneously, thereby highlighting their ability to capture and leverage shared information across different attributes. We replicate the experiment setup from~\cite{liu2024famo,ban2024fair}. The model consists of a 9-layer convolutional neural network (CNN) as the backbone, with a dedicated linear layer for each objective. Training is performed over 15 epochs, utilizing the Adam optimizer~\cite{kingma2014adam} with a learning rate of $3e-4$, and a batch size of 256.

In Table~\ref{tab:cityscapes-and-celeba}, we observe that \psmg{} demonstrates commendable performance in the final results of CityScapes and CelebA. In particular, in CelebA, while \dm{} of \psmg{} ranks 6th best, its \mr{} ranks the highest among all listed MOO algorithms. This indicates that \psmg{} exhibits a balanced capability in dealing with a significant number of objectives concurrently.

 \begin{table*}[t!]
    \centering
    \resizebox{\textwidth}{!}{%
    \begin{tabular}{lrrrrrrrr}
    \toprule
   \multirow{4}{*}{\textbf{Method}} & \multicolumn{6}{c}{\textbf{CityScapes}} & \multicolumn{2}{c}{\textbf{CelebA}} \\
      \cmidrule(lr){2-7}\cmidrule(lr){8-9}
      &  \multicolumn{2}{c}{Segmentation} & \multicolumn{2}{c}{Depth} & \multirow{2}{*}{\mr{} $\downarrow$} & \multirow{2}{*}{ \dm{} $\downarrow$}  & \multirow{2}{*}{\mr{} $\downarrow$} & \multirow{2}{*}{ \dm{} $\downarrow$}\\
    \cmidrule(lr){2-3}\cmidrule(lr){4-5}
     &  mIoU $\uparrow$ & Pix Acc $\uparrow$ & Abs Err $\downarrow$ & Rel Err $\downarrow$ &  &  & &\\
    \midrule 
    \stl{}      & 74.01 & 93.16 & 0.0125 & 27.77 & & \\
    \midrule
    \ls{}       & 70.95 & 91.73 & 0.0161 & 33.83 & 10.00 & 22.60 & 7.55 & 4.15 \\
    \si{}       & 70.95 & 91.73 & 0.0161 & 33.83 & 12.50 & 14.11 & 9.15 & 7.20 \\
    \rlw{}      & 74.57 & 93.41 & 0.0158 & 47.79 & 12.75 & 24.38 & 6.38 & 1.46 \\
    \dwa{}      & 75.24 & 93.52 & 0.0160 & 44.37 & 9.75 & 21.45 & 8.25 & 3.20 \\
    \uw{}       & 72.02 & 92.85 & 0.0140 & \best{30.13} & 8.75 & 5.89 & 6.80 & 3.23\\
    \mgda{}     & 68.84 & 91.54 & 0.0309 & 33.50 & 13.00 & 44.14 & 12.80 & 14.85 \\
    \pcgrad{}   & 75.13 & 93.48 & 0.0154 & 42.07 & 10.25 & 18.29 & 7.95 & 3.17 \\
    \graddrop{} & 75.27 & 93.53 & 0.0157 & 47.54 & 9.00 & 23.73 & 9.30 & 3.29 \\
    \cagrad{}   & 75.16 & 93.48 & 0.0141 & 37.60 & 8.75 & 11.64 & 7.35 & 2.48 \\
    \imtlg{}    & 75.33 & 93.49 & 0.0135 & 38.41 & 6.50 & 11.10 & 5.67 & 0.84 \\
    \moco{}    & 75.42 & 93.55 & 0.0149 & 34.19 & 5.25 & 9.90 & $\times$ & $\times$ \\
    \nashmtl{}  & 75.41 & 93.66 & \best{0.0129} & 35.02 & 3.50 & 6.82 & 5.97 & 2.84 \\
    % \midrule 
    \famo{}     & 74.54 & 93.29 & 0.0145 & 32.59 & 9.00  & 8.13 & 5.75 & 1.21 \\
    \fairgrad{}     & \best{75.72} & \best{93.68} & 0.0134 & 32.25 & \best{1.50} & \best{5.18} &  6.47 & \best{0.37} \\
    \sdmgrad{}     &  74.53 & 93.52 & 0.0137 & 34.01 & 7.50  & 7.79 & $\times$  & $\times$\\
    \midrule
    \psmg{}     & 74.90 & 93.37 & 0.0135 & 35.78 & 8.50 & 8.76 & \best{5.60} & 2.60 \\
    \bottomrule 
    \end{tabular}
    }
    \vspace{2pt}
    \caption{Results for the CityScapes (2 tasks) and CelebA (40 tasks) datasets. We run each experiment three time using random seeds with the mean reported. The best result is highlighted in bold.}
    \label{tab:cityscapes-and-celeba}
    \vspace{-20pt}
\end{table*}

\subsection{Ablation on $R$}

To better analyze the effect of the hyperparameter $R$, we present the ablation study on the QM-9 dataset. We choose $R \in \{4, 8, 16, 32\}$ and momentum $\Blambda=0.9$ in averaged over 3 random seeds.

In Figure~\ref{fig:QM9_abloss}, it can be observed that the test loss increases with a higher hyperparameter $R$. Nevertheless, there is a minimal difference in the final convergences between \textsc{PSMGD-4} and \textsc{PSMGD-8}, indicating proximity to the best convergence with $R=8$. For comparison, \ls{} and \mgda{} are also included in Table~\ref{tab:qm9_ab}. We note that the training time per epoch will rise as $R$ decreases, indicating that solving the optimization problem ~\ref{eq: QP} will result in increased time consumption. Among the final results in \dm{}, \textsc{PSMGD-4} demonstrates superior performance with the longest training time in \psmg{} series algorithms. While \textsc{PSMGD-16} and \textsc{PSMGD-32} have a shorter training duration, their \dm{} values do not match those of \textsc{PSMGD-8}.

\begin{figure}[htbp]
    \centering
    \includegraphics[width=\columnwidth]{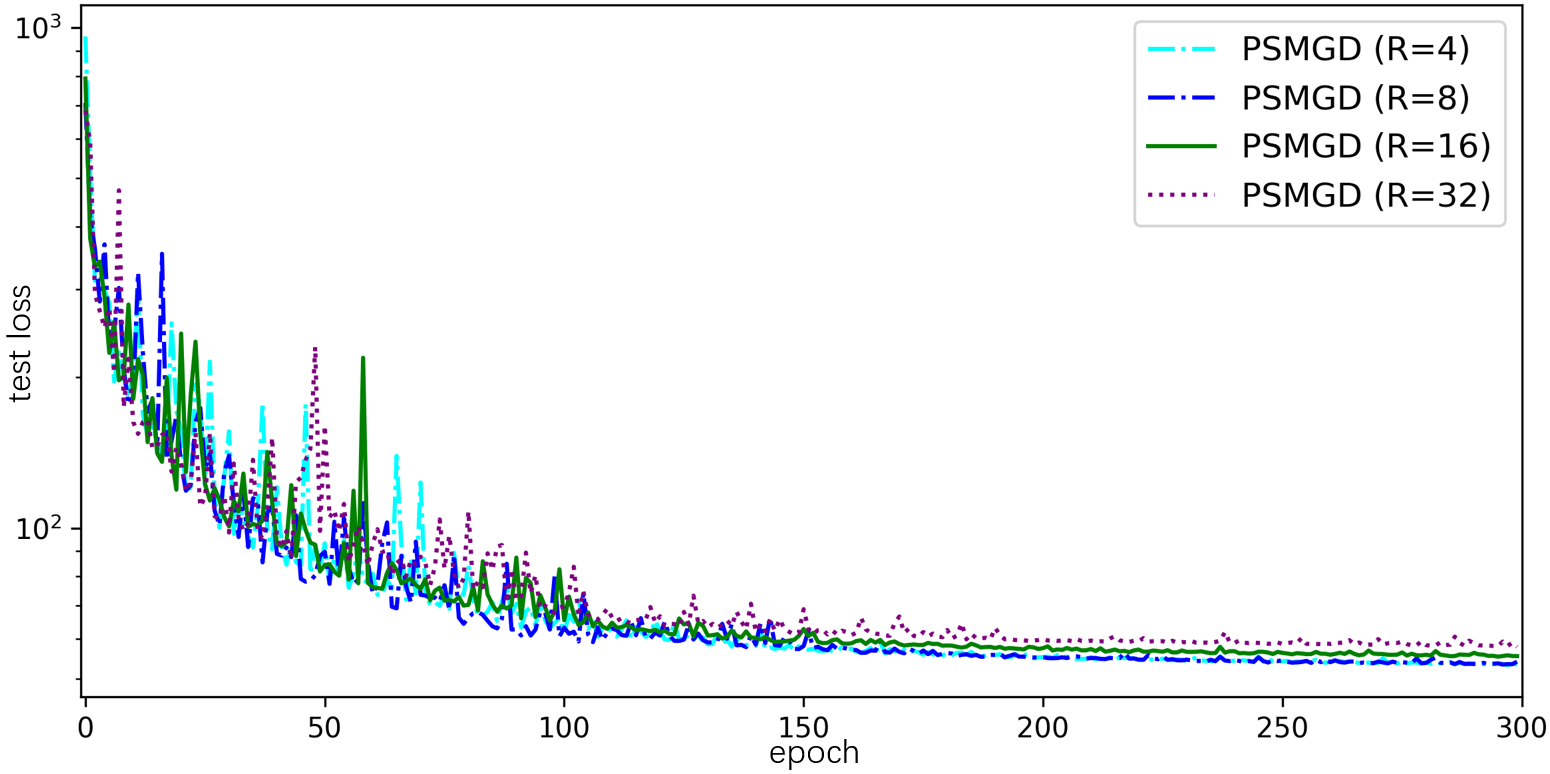}

    \caption{Test loss in training (300 epochs) on QM-9.
    }
    \label{fig:QM9_abloss}
\end{figure}

\begin{table}[htbp]
    \centering
    \resizebox{\columnwidth}{!}{%
    \begin{tabular}{lcr}
    \toprule
    \multirow{2}{*}{\textbf{Method}} & Training time $\downarrow$ & \multirow{2}{*}{ \dm{} $\downarrow$}\\
    & (mean per epoch) \\ 
    \midrule
    \ls{}   & \best{1.70} & 177.6\\
    \mgda{}      & 12.94 & 120.5\\
    % \famo{} & 1.96 & \best{58.5}\\
    \midrule
    \textsc{PSMGD-4}   & 2.61 & \best{88.2}\\
    \textsc{PSMGD-8}   & 2.23 & 92.4\\
    % \midrule
    \textsc{PSMGD-16}   & 2.13 & 109.7\\
    \textsc{PSMGD-32} & 2.07 & 116.2\\
    % \imtlg{} & 5.38\\
    % \nashmtl{}        & 8.45\\
    % \fairgrad{}        & 5.40\\
    % \midrule
    % \psmg{}        & 2.23\\
    \bottomrule
    \end{tabular}
    }
    \vspace{2pt}
    \caption{Training time per epoch [Min.] (averaged over 3 random seeds) on QM-9.}
    \label{tab:qm9_ab}
    \vspace{-10pt}
\end{table}

\end{document}